\documentclass[twoside]{article}

\usepackage[accepted]{aistats2024}
\usepackage[round]{natbib}

\usepackage[colorlinks,
            linkcolor=red,
            anchorcolor=blue,
            citecolor=blue, 
            pagebackref=true
           ]{hyperref}
\usepackage{mmll}
\usepackage{tikz}
\usetikzlibrary{arrows,automata,positioning}
\usepackage{todonotes}
\usepackage{setspace}
\usepackage{tabularx}
\usepackage{float}
\usepackage{enumitem}
\usepackage{wrapfig,lipsum}
\usepackage{thm-restate}

\DeclareMathOperator*{\argsup}{arg\,sup} 
\DeclareMathOperator*{\arginf}{arg\,inf}

\DeclareMathOperator*{\dist}{dist} 
\newcommand{\algname}{\text{DR-LSVI-UCB}}
\newcommand{\intextmath}[1]{{\small #1}}
\crefname{assumption}{assumption}{assumptions}

\makeatletter
\renewcommand{\todo}[2][]{%
    \@todo[caption={#2}, #1]{\begin{spacing}{0.5}#2\end{spacing}}%
} 
\makeatother 

\renewcommand*{\backref}[1]{}
\renewcommand*{\backrefalt}[4]{%
\ifcase #1 %
    No citations.%
\or
    (p. #2.)%
\else
    (pp. #2.)%
\fi}%

\usepackage[compact]{titlesec}
\usepackage{sidecap}
\titlespacing*{\section}{0pt}{*0.1}{*0.1}
\titlespacing*{\subsection}{0pt}{*0.1}{*0.1}
\titlespacing*{\subsubsection}{0pt}{*0.1}{*0.1}

\allowdisplaybreaks

\begin{document}

\runningtitle{Distributionally Robust Off-Dynamics Reinforcement Learning}

\twocolumn[

\aistatstitle{Distributionally Robust Off-Dynamics Reinforcement Learning: Provable Efficiency with Linear Function Approximation}

\aistatsauthor{ Zhishuai Liu \And  Pan Xu }

\aistatsaddress{ 
Duke University \\
zhishuai.liu@duke.edu

\And  
Duke University \\
pan.xu@duke.edu
} 
]

\begin{abstract}
  We study off-dynamics Reinforcement Learning (RL), where the policy is trained on a source domain and deployed to a distinct target domain. We aim to solve this problem via online distributionally robust Markov decision processes (DRMDPs), where the learning algorithm actively interacts with the source domain while seeking the optimal performance under the worst possible dynamics that is within an uncertainty set of the source domain's transition kernel. We provide the first study on online DRMDPs with function approximation for off-dynamics RL. We find that DRMDPs' dual formulation can induce nonlinearity, even when the nominal transition kernel is linear, leading to error propagation. By designing a $d$-rectangular uncertainty set using the total variation distance, we remove this additional nonlinearity and bypass the error propagation. We then introduce DR-LSVI-UCB, the first provably efficient online DRMDP algorithm for off-dynamics RL with function approximation, and establish a polynomial suboptimality bound that is independent of the state and action space sizes. Our work makes the first step towards a deeper understanding of the provable efficiency of online DRMDPs with linear function approximation. Finally, we substantiate the performance and robustness of DR-LSVI-UCB through different numerical experiments.
\end{abstract}
\section{INTRODUCTION}

The Markov decision process (MDP) is a prevalent model in dynamic decision-making and reinforcement learning \citep{puterman2014markov,sutton2018reinforcement}. A central challenge in employing MDPs in various applications lies in the lack of knowledge of model parameters, notably the transition kernels. Existing studies mostly hinge on the assumption that the environment in which a policy is trained is identical to that in which it is deployed. However, in practical scenarios where this assumption is violated, standard RL methods are prone to severe failures \citep{farebrother2018generalization, packer2018assessing, zhao2020sim}, a phenomenon known as the {\it sim-to-real gap}. Infectious disease control \citep{laber2018optimal, liu2023deep} exemplifies such a case wherein an agent trains policies on simulators extensively utilized in environmental studies. Nonetheless, these simulators cannot fully capture the environmental evolution complexity, and environmental changes may also occur over time, further contributing to the sim-to-real gap. Another instance is found in robotics learning, where slight variations between training and testing environments, such as terrain or target parameters, may lead to task failure \citep{maitin2010cloth, tobin2017domain, peng2018sim}.

Learning under the sim-to-real gap can be conceptualized as an off-dynamics RL problem \citep{koos2012transferability,wulfmeier2017mutual,eysenbach2020off,jiang2021simgan}, where an agent trains a policy in an accessible source domain, such as a simulator or the present environment, then deploys the learned policy in a distinct target domain, which could be the real environment the agent encounters during operation or a future changing environment. The dynamics shift between environments necessitates a robust strategy for policy learning in the source domain, ensuring that the policy can work effectively in different yet structurally similar target domains.

Distributionally robust Markov decision process (DRMDPs) \citep{satia1973markovian, nilim2005robust, iyengar2005robust} address the sim-to-real gap challenge by modeling the uncertainty of transition kernels. It aims to learn a robust policy that performs well under the worst-case transition kernel within the uncertainty set defined based on the source environment \citep{xu2006robustness, wiesemann2013robust, zhang2021robust, yang2022toward, panaganti2022robust, shi2022distributionally,yang2023distributionally,shen2024wasserstein}. Existing DRMDP research can be categorized based on the assumption on the source domain: (i) planning problems where the exact model is assumed known, (ii) learning under a generative model, and (iii) learning from offline datasets utilizing specific data coverage assumptions. However, in practice, formulating and solving a planning problem is often infeasible due to imperfect knowledge or complexity of the source domain. Similarly, an accurate generative model representing the source domain is usually unavailable. Additionally, most data coverage assumptions require the datasets have sufficient coverage of distributions induced by the optimal policy under any transition kernel in the uncertainty set. 
Since the optimal policy is usually unknown and there are infinite number of transition kernels in the uncertainty set, practical verification of data coverage assumptions is intractable. %
Thus, when incremental collection of data  through active interactions with the source domain is feasible, online algorithms  without relying on additional oracles or data coverage assumptions about the optimal policy will be preferred. We refer to this as the online DRMDP problem.

Another significant challenge in RL is the ubiquitous presence of applications with arbitrarily large state and action spaces, which require suitable function approximations to alleviate the curse of dimensionality. Although approaches based on linear function approximation have exhibited theoretical and empirical success in numerous settings under standard MDP \citep{bhandari2018finite, modi2020sample, jin2020provably, he2023nearly, he2021logarithmic, yang2020reinforcement}, DRMDP encounters additional difficulties when combined with linear function approximations since the dual formulation in worst-case analyses may induce extra non-linearity, even when the source domain transition kernel is linear \citep{tamar2014scaling, pinto2017robust, derman2018soft, mankowitz2019robust, derman2020bayesian,  zhang2021robust, badrinath2021robust}. Consequently, the theoretical understanding of online DRMDPs with function approximation  remains elusive, even when the approximation is linear. This leads to the open question: %
\begin{center}
\vspace{-0.1in}
{\it When is it possible to design a provably efficient algorithm for online DRMDPs \\ with linear function approximation?}  
\vspace{-0.1in}
\end{center}

In this work, we provide the first analysis of online DRMDP with linear function approximation where an agent actively interacts with the source domain to learn a robust policy.

\textbf{Our main contributions} are summarized as follows.
\begin{itemize}[nosep,leftmargin=*]
    \item We first investigate the differences in applying linear function approximation in DRMDPs with uncertainty sets defined on different probability divergence metrics. We show that the strong duality for Chi-square or Kullback-Leibler (KL) based DRMDPs induces additional nonlinearity which can cause severe error amplification and regret accumulation
    (see \Cref{remark:linear regression} for more details).
    We then identify a feasible setting that assumes a $d$-rectangular linear DRMDP and a total variation (TV) based uncertainty set, which permits linear representations on the robust Q-functions, and bypasses the error amplification and regret accumulation. 
    
    \item We introduce a model-free online algorithm, viz., DR-LSVI-UCB, based on the LSVI-UCB algorithm in the non-robust setting \citep{jin2020provably}. The design of the DR-LSVI-UCB incorporates a robust Upper Confidence Bonus (UCB) quantity and a truncated estimation of the robust state-action value function at the MDP's fail state, both of which are explicitly devised for the online DRMDP setting (refer to \Cref{remark:alg-8&9} for more details). 
    
    \item We prove an average suboptimality bound for DR-LSVI-UCB in the order of $\tilde O(\sqrt{H^4d^4/K})$, where $H$ is the horizon length, $d$ the feature dimension, and $K$ the number of episodes. Our result matches the average regret\footnote{Since in DRMDP, we trade off the performance in the source domain for the robustness in the target domain, we evaluate a robust algorithm by its suboptimality gap from the optimal robust policy, comparable to the average regret in standard MDP, i.e., the cumulative regret divided by $K$.} bound of its non-robust counterpart LSVI-UCB \citep{jin2020provably} regarding $H$ and  $K$, but is worse regarding feature dimension by a factor of $\sqrt{d}$. 
    To the best of our knowledge, this is the first non-asymptotic suboptimality bound for online DRMDPs with linear function approximation, which guarantees efficient robust learning in off-dynamics RL.  
    Interesting, when reduced to the tabular setting where $d=SA$ with $S$ and $A$ being the state and action space sizes, the average suboptimality gap of \algname\ exactly matches the average regret bound of LSVI-UCB, indicating tabular DRMDPs with a TV uncertainty set might not be more challenging than the standard tabular MDP.
    
    \item We perform numerical experiments to illustrate the efficacy of \algname\ on a simulated linear MDP environment and an emulated American put option environment \citep{tamar2014scaling}. Our results demonstrate that the policies derived by \algname\  are robust against dynamics shifts, further substantiating our theoretical findings.
\end{itemize}

\section{RELATED WORK}
\paragraph{Episodic Linear MDP} Our study focuses on the episodic linear MDP setting. Specifically, we assume the nominal transition probability in our DRMDP admits the linear MDP structure. There has been a recent surge in research on episodic linear MDPs \citep{yang2020reinforcement, jin2020provably, modi2020sample, zanette2020frequentist, wang2020reward, he2021logarithmic, wagenmaker2022reward, ishfaq2023provable,he2023nearly}.  The most relevant study to ours is the seminal work of \cite{jin2020provably}, which introduced a model-free online algorithm, LSVI-UCB, for standard RL. Through a `Hoeffding-type' exploration bonus, LSVI-UCB can actively explore the nominal environment and achieves a $\tilde{O}(\sqrt{d^3H^4K})$ regret bound. However, the episodic linear MDP setting still remains understudied in the context of DRMDPs. 

\paragraph{DRMDPs} 
Numerous works have extensively studied the DRMDP framework under different settings. \cite{xu2006robustness, wiesemann2013robust, yu2015distributionally, mannor2016robust, goyal2023robust} studied the DRMDP assuming the exact environment is known, and establishing DRMDPs as classic planning problems. \cite{zhou2021finite, yang2022toward, panaganti2022sample, xu2023improved, shi2023curious, yang2023avoiding} studied the DRMDP assuming the access to a generative model. \cite{panaganti2022robust, shi2022distributionally, blanchet2023double} studied the DRMDP in the offline RL setting assuming strong data coverage or concentratability conditions. Moreover, \cite{dong2022online} studied the online DRMDP under the episodic tabular MDP setting. They proposed a model-based algorithm ROPO, which achieves an average suboptimality bound of $\tilde{O}(\sqrt{H^4S^2A/K})$ under the $(s,a)$-rectangular assumption. However, their method cannot deal with settings where state space size $S$ and action space size $A$ are large or infinite in practical applications.

\paragraph{DRMDPs with linear function approximation} \cite{tamar2014scaling} first proposed to use linear function approximation to solve DRMDPs with large state and action spaces, and provided an asymptotic convergence guarantee for their sampling-based approach. \cite{badrinath2021robust} proposed a model-free online algorithm based on linear projection, and provided the corresponding asymptotic convergence guarantee. Recently, \cite{ma2022distributionally} pointed out that the nonlinearity of DRMDPs might make linear projection fall short, resulting in poor decision-making. \cite{ma2022distributionally} then studied the novel $d$-rectangular linear DRMDP that naturally admits linear representations of the robust state-action value function. 
They studied the offline setting and proposed two value iteration based algorithms under the uniformly well-explored dataset assumption and the sufficient coverage of the optimal policy assumption, respectively. \cite{blanchet2023double} also studied the offline $d$-rectangular linear DRMDP based on the robust partial coverage assumption. However, the data coverage assumptions cannot be verified and guaranteed in practice as we discussed in \Cref{rmk:compare_with_data_coverage}. Thus, an online algorithm, which automates the acquisition of the optimal robust policy through actively interacting with the source domain, for the episodic $d$-rectangular linear DRMDP is in need.

\section{DISTRIBUTIONALLY ROBUST MDP WITH LINEAR FUNCTION APPROXIMATION}\label{sec:DRMDP and Linear MDP}
\subsection{Preliminaries}
A finite horizon Markov decision process can be denoted as $\text{MDP}(\mathcal{S}, \mathcal{A}, H,  P, r)$. Here $\mathcal{S}$ and $\mathcal{A}$ are the state and action spaces, 
$H\in \ZZ_+$ is the horizon length, 
$P=\{P_h\}_{h=1}^{H}$ and $r=\{r_h\}_{h=1}^H$ are the set of transition kernels and reward functions, respectively. For each step $h\in[H]$, we denote $P_{h}(\cdot|s,a)$ as the transition probability measure over the next state if action $a$ is taken at state $s$, and 
$r_h: \cS\times\cA \rightarrow [0,1]$ is the deterministic reward function, which for simplicity is assumed to be known.

A non-stationary Markov policy $\pi=\{\pi_h\}_{h=1}^H$ is a sequence of decision rules, where $\pi_h: \cS \rightarrow \Delta(\cA)$ is the policy at step $h$ and $\Delta(\cA)$ is the probability simplex defined over the action space $\cA$. For any transition kernel $P$ and any policy $\pi$, we define the value function and the state-action value function (viz., the Q-function) at step $h$ as
{\small
\begin{align*}
    & {\textstyle V_h^{\pi, P}(s):=\mathbb{E}^P\big[\sum_{t=h}^H r_t(s_t,a_t)\big|s_h=s, \pi \big]},\\ 
    &{\textstyle Q_h^{\pi, P}(s,a):=\mathbb{E}^P\big[\sum_{t=h}^H r_t(s_t,a_t)\big|s_h=s,a_h=a, \pi \big]}.
\end{align*}}%
As the rewards are bounded in $[0,1]$, thus any value function and Q-function are bounded in $[0, H]$.

A finite horizon distributionally robust Markov decision process (DRMDP) is formally defined by a tuple $\text{DRMDP}(\mathcal{S}, \mathcal{A}, H, \mathcal{U}^{\rho}(P^0), r)$. Here, $P^0=\{P^0_h\}_{h=1}^H$ is the set of nominal transition kernels, and $\cU^{\rho}(P^0) =  \bigotimes_{h\in[H]}\cU^{\rho}_h(P_h^0)$ denotes an uncertainty set centered around the nominal transition kernel with an uncertainty level $\rho\geq 0$.  $\cU^{\rho}_h(P_h^0)$ is often defined as a ball centered around $P_h^0$ with radius $\rho$ based on different probability divergence measures \citep{iyengar2005robust, yang2022toward, xu2023improved}. 

In contrast with the standard MDP where only the nominal transition kernel $P^0$ is considered, in DRMDPs, we consider all transition kernels within the uncertainty set $\cU^{\rho}(P^0)$. Then for $h \in [H]$ and any policy $\pi$, we define the robust value function $V_h^{\pi, \rho}:\cS\rightarrow \RR$ as the value function under the worst possible transition kernel within the uncertainty set: 
{\small
\begin{align*}
   {\textstyle V_h^{\pi, \rho}(s) =\inf_{P \in \mathcal{U}^{\rho}(P^0)}V_{h}^{\pi, P}(s), \quad \forall (h,s) \in [H] \times \cS.} 
\end{align*}}%
Accordingly, we define the robust state-action value function as  
${\intextmath{\textstyle Q_h^{\pi, \rho}(s,a) =\inf_{P \in \mathcal{U}^{\rho}(P^0)}Q_{h}^{\pi, P}(s,a)} }$,
for any $(h,s,a) \in [H]\times \mathcal{S} \times \mathcal{A}$.

We then define the optimal robust value function and optimal robust state-action value function:  $\forall(h,s,a) \in [H]\times \mathcal{S} \times \mathcal{A}$, $V_h^{\star, \rho}(s) = \sup_{\pi \in \Pi} V_h^{\pi, \rho}(s)$, $
    Q_h^{\star, \rho}(s,a) = \sup_{\pi \in \Pi} Q_h^{\pi, \rho}(s,a)$,
where $\Pi$ is the set of all (possibly randomized and nonstationary) policies. Then the optimal robust policy  $\pi^{\star}=\{\pi^{\star}_h\}_{h=1}^H$, defined as the policy that achieves the optimal robust value function, is given by $\pi^{\star} = \argsup_{\pi \in \Pi}V_h^{\pi, \rho}(s)$, for any $(h,s) \in [H]\times \mathcal{S}$.
Our goal is to learn the optimal robust policy by actively interacting with the nominal environment within $K$ episodes. At the beginning of episode $k$, the agent receives an initial state $s_1^k$. Denote $\pi^k$ as the current policy of the agent. We use \intextmath{$V_1^{\star, \rho}(s_1^k)-V_1^{\pi^k, \rho}(s_1^k)$} to measure the suboptimality of policy $\pi^k$ at episode $k$. Hence, we are interested in the average suboptimality of an algorithm after $K$ episodes, i.e., $\text{AveSubopt}(K)$, defined as follows
{\small
\begin{align*}
   {\textstyle \text{AveSubopt}(K) = \frac{1}{K}\sum_{k=1}^K\big[V_1^{\star,\rho}(s_1^k) - V_1^{\pi^k, \rho}(s_1^k)\big].} 
\end{align*}}%

\subsection{$d$-Rectangular Linear DRMDP}
\label{sec:DRMDP}
In this paper, we define the uncertainty set $\cU_h^\rho(P_h^0)$ based on a linear structure of the nominal transition kernel $P_h^0$, called the linear MDP \citep{jin2020provably, wei2021learning, wagenmaker2022reward, he2023nearly}.

\begin{assumption}
\label{assumption:linear MDP}
(Linear MDP) Given a known state-action feature mapping $\bphi: \cS \times \cA \rightarrow \RR^d$ satisfying $ \sum_{i=1}^d\phi_i(s,a)=1$, $\phi_i(s,a) \geq 0$,
for any $(i, s,a)\in [d] \times \cS \times \cA$,
we assume the reward function $\{r_h\}_{h=1}^H$ and nominal transition kernels $\{P_h^0\}_{h=1}^H$ have linear structures. Specifically, for any $(h,s,a) \in [H]\times \mathcal{S} \times \mathcal{A}$, $r_h(s,a)=\langle \bphi(s,a), \btheta_h\rangle$, and $P_h^0(\cdot|s,a)=\langle \bphi(s,a), \bmu_h^0(\cdot)\rangle$,
where $\{\btheta_h\}_{h=1}^H$ are known vectors with bounded norm $\Vert \btheta_h \Vert_2 \leq \sqrt{d}$ and $\{\bmu_h\}_{h=1}^H$ are unknown probability measures over $\cS$. 
\end{assumption}
\Cref{assumption:linear MDP} is slightly stronger than the linear MDP studied in the standard RL literature. Following similar works in DRMDPs \citep{ma2022distributionally,blanchet2023double}, 
we assume the coordinates of the feature mapping $\bphi(\cdot, \cdot)$ to be positive and add up to one, which could be achieved by normalization. Meanwhile, the factor measures $\{\bmu_h\}_{h=1}^H$ are required to be proper probability measures. Under these additional constraints, the nominal transition kernel $P^0_h(\cdot|s,a)$ can be seen as a mixture of factor distributions $\bmu_h(\cdot)$ with the aggregated feature $\phi(s,a)$ determining the weights.

To incorporate the linear structure of $P^0$ into the uncertainty set $\cU_h^\rho (P_h^0)$, we adopt the notion of $d$-rectangular uncertainty set \citep{ma2022distributionally,goyal2023robust}. More specifically, we assume $\cU^{\rho}(P^0)$ is parameterized by $\{\bmu_h^0\}_{h=1}^H$ and can be decomposed into \intextmath{$\cU^{\rho}_h(P_h^0) = \bigotimes_{(s,a)\in \cS\times\cA}\cU_h^{\rho}(s,a; \bmu^0_h)$}, where \intextmath{$\cU_h^{\rho}(s,a; \bmu^0_h)=\{\sum_{i=1}^d \phi_i(s,a)\mu_{h,i}(\cdot): \mu_{h,i}(\cdot) \in \cU_{h, i}^{\rho}(\mu^0_{h,i}), \forall i \in [d]\}$}, 
and \intextmath{$\cU_{h,i}^{\rho}(\mu^0_{h,i})$} is defined as 
{\small
\begin{align}
    \label{eq:factor uncertainty set}
    {\textstyle \cU_{h, i}^{\rho}(\mu^0_{h,i}) = \big\{\mu: \mu\in \Delta(\cS), D(\mu||\mu_{h,i}^0)\leq \rho\big\}}.
\end{align}}%
Here $D(\cdot||\cdot)$ is a probability divergence metric that will be instantiated later. 
We remark that the factor uncertainty sets $\{\cU_{h,i}^{\rho}(\mu_{h,i}^0)\}_{i\in[d]}$ are independent of the state-action pair $(s,a)$, and also independent with each other. 
As we will show in the proof of \Cref{prop:Linear form}, these attributes are essential in deriving that, for all policies, the robust Q-functions are always linear in the feature mapping $\bphi(\cdot,\cdot)$.

\subsection{Robust Bellman Equation and the Optimal Policy in DRMDPs} 
We show that the robust value function and the robust Q-function defined in DRMDPs satisfy the following robust Bellman equation. We denote $[\PP_hV](s,a) = \EE_{s'\sim P_h(\cdot|s,a)}[V(s')]$ for simplicity.

\begin{proposition}
\label{prop:Robust Bellman equation}
(Robust Bellman equation) Under the $d$-rectangular linear DRMDP setting, for any nominal transition kernel $P^0=\{P^0_h\}_{h=1}^H$ and any stationary policy $\pi=\{\pi_h\}_{h=1}^H$, the following robust Bellman equation holds: for any $(h,s,a) \in [H]\times \cS \times \cA$,  %
{\small
\begin{align}
\label{eq:robust bellman equation}
& Q_h^{\pi, \rho}(s,a)=r_h(s,a)+\inf_{P_h(\cdot|s,a)\in\cU_h^{\rho}(s,a;\bmu_h^0)}[\PP_h V_{h+1}^{\pi,\rho}](s,a)\notag\\
&{\textstyle V_h^{\pi, \rho}(s) =\EE_{a\sim\pi_h(\cdot|s)}\big[Q_h^{\pi,\rho}(s,a)\big]}.
\end{align}}%
\end{proposition}

Furthermore, it is well-known that the optimal (robust) value function can be achieved by a deterministic and stationary policy in standard MDPs \citep{sutton2018reinforcement, agarwal2019reinforcement} and tabular DRMDPs with $(s,a)$-rectangular assumption \citep{iyengar2005robust, nilim2005robust}. Similarly, we show that the optimal robust value function and Q-function can be achieved by a deterministic and stationary policy $\pi^{\star}$ in the $d$-rectangular linear DRMDP. 

\begin{proposition}
\label{prop:Deterministic and stationary}
(Existence of the optimal policy) Assume the nominal transition kernel $P^0$ satisfies \Cref{assumption:linear MDP} and the uncertainty set $\cU^\rho(P^0)$ is defined as in \Cref{sec:DRMDP}. Then there exists a deterministic and stationary policy $\pi^{\star}$ such that \intextmath{$ V_h^{\pi^{\star}, \rho}(s)=V^{\star, \rho}_h(s)$} and \intextmath{$Q_h^{\pi^{\star}, \rho}(s,a)=Q_h^{\star, \rho}(s,a)$}, for any $(h,s,a)\in[H]\times\cS\times\cA$.
\end{proposition}

The results in \Cref{prop:Robust Bellman equation,prop:Deterministic and stationary} have been used in existing analyses of DRMDPs without proof \citep{ma2022distributionally,blanchet2023double}. For completeness, we provide their proofs in \Cref{sec:proof of prop}. With these results, we can safely restrict the policy class $\Pi$ to the deterministic and stationary one. This  leads to the robust Bellman optimality equation:
{\small
\begin{align}
\label{eq:optimal robust bellman equation}
     Q_h^{\star, \rho}(s,a) &= r_h(s,a) +\inf_{P_h(\cdot|s,a) \in \cU_h^{\rho}(s,a;\bmu_h^0)}[\PP_h V_{h+1}^{\star, \rho}](s,a)\notag\\
    {\textstyle V_h^{\star, \rho}(s)}&={\textstyle \max_{a\in\cA}Q_h^{\star}(s,a)}.
\end{align}}%
\eqref{eq:optimal robust bellman equation} suggests that the optimal robust policy is greedy with respect to the optimal robust Q-function. Therefore, it suffices to estimate $Q_h^{*,\rho}$ to find $\pi^*$.

\subsection{DRMDPs with TV divergence}
In this work, we focus on the total variation (TV) distance as the probability divergence metric employed in defining the uncertainty set \eqref{eq:factor uncertainty set}. Given any two probability distributions $P$ and $Q$, the TV divergence, denoted by $D_{TV}(P\Vert Q)$, can be expressed as 
{\small
\begin{align}
\label{def:TV-divergence}
{\textstyle D_{TV}(P\Vert Q)=1/2\int_{\cS}|P(s)-Q(s)|ds}.
\end{align}}%
The optimization problem in \eqref{eq:robust bellman equation} has the following dual formulation under the TV uncertainty set.
\begin{proposition}
\label{prop:strong duality for TV}
    (Strong duality for TV \citep[Lemma 4]{shi2023curious}). Given any probability measure $\mu^0$ over $\cS$, a fixed uncertainty level $\rho$, the uncertainty set $ \cU^{\rho}(\mu^0) =\{\mu: \mu\in \Delta(\cS), D_{TV}(\mu||\mu^0)\leq \rho\}$, and any function $V:\cS \rightarrow [0,H]$, we obtain 
    {\small
    \begin{align}\label{eq:duality}
        &{\textstyle \inf_{\mu\in\cU^{\rho}(\mu^0)}\EE_{s\sim\mu}V(s)} = \max_{\alpha \in [V_{\min}, V_{\max}]}\big\{\EE_{s\sim \mu^0}[V(s)]_{\alpha} \notag\\
        &\qquad {\textstyle-\rho\big(\alpha - \min_{s'}[V(s')]_{\alpha}\big) \big\}},
    \end{align}}%
    where $[V(s)]_{\alpha}=\min\{V(s), \alpha\}$, $V_{\min}=\min_{s}V(s)$ and $V_{\max}=\max_{s}V(s)$. Notably, the range of $\alpha$ can be relaxed to $[0,H]$ without impacting the optimization. 
\end{proposition}

\section{ROBUST LEAST SQUARE VALUE ITERATION WITH UCB EXPLORATION}\label{sec:linear representation of dual and the algorithm}

\subsection{Linear Representation of the Robust State-Action Value Function}
Recall the strong duality in \eqref{eq:duality}, we need to solve the minimization problem, $\min_{s'}[V(s')]_{\alpha}$, which is challenging when it is not convex with respect to $s'$ and computationally inefficient when  $\cS$ is large. To overcome this issue, we make the same fail-state assumption made in the function approximation setting \citep{panaganti2022robust} and show that it is compatible with the $d$-rectangular linear DRMDP.
\begin{assumption}
\label{assumption:fail-state}
    (Fail-state) The linear MDP has a `fail state' $s_f$, such that for all $ (h,a) \in [H]\times \cA$, $r_h(s_f, a)=0$, $P_h^0(s_f|s_f,a)=1$.
\end{assumption}

The existence of fail states is natural in many real-world applications such as the collapse of a robot in robotics \citep{panaganti2022robust}. As another example in the context of cancer treatments, patients could die, or the cancer may advance further, during the course of a finite-stage treatment process \citep{goldberg2012q, zhao2018constructing, Liu2023dtr}, both of which could be considered as fail states.

We show that \Cref{assumption:fail-state} is compatible with the linear MDP structure.
In particular, we show that we can extend the original $d$-rectangular linear DRMDP as follows. First, we define a new feature mapping $\tilde{\bphi}: \cS\times\cA\rightarrow \RR^{d+1}$ based on the original one:
{\small
\begin{align*}
    &{\textstyle \tilde{\bphi}(s_f,a)=[1, 0, \cdots, 0]^{\top}, \quad \forall a\in \cA},\\
    &{\textstyle \tilde{\bphi}(s,a)=\big[0, \bphi(s,a)^{\top}\big]^{\top}, \quad \forall (s,a)\in \cS/\{s_f\}\times\cA}.
\end{align*}}%
It is easy to verify that
{\small $\sum_{i=1}^{d+1}\tilde{\phi}_i(s,a)=1$, $\forall (s,a)\in\cS\times\cA$}, and $\tilde{\phi}_i(s,a) \geq 0$, $\forall i\in[d+1]$. Let {\small$\tilde{\btheta}_h = [0, \btheta_h^{\top}]^{\top}$}, and {\small$\tilde{\bmu}_h^0(\cdot)=[\delta_{s_f}(\cdot), \bmu_h^0(\cdot)^{\top}]^{\top}$}, where $\delta_{s_f}$ is the Dirac delta distribution with mass at $s_f$. We can show that the reward functions and transition kernels are still linear based on the new notations. Then we can define the same $d$-rectangular uncertainty set as in \Cref{sec:DRMDP}. For simplicity, we assume the fail-state assumption holds in the original linear MDP in this paper.

\begin{remark}
\label{remark:dual with fail state}
     Under \Cref{assumption:fail-state}, \Cref{prop:strong duality for TV} can be further simplified. For any function $V:\cS \rightarrow [0,H]$ with $\min_{s\in\cS}V(s)=V(s_f)=0$, we have $\inf_{\mu\in\cU^{\rho}(\mu^0)}\EE_{s\sim\mu}V(s) = \max_{\alpha \in [0,H]}\{\EE_{s\sim \mu^0}[V(s)]_\alpha - \rho\alpha \}$.
Then with the fail state $s_f$, for any $(\pi, h,a) \in \Pi\times [H]\times \cA$, we have $Q^{\pi, \rho}_h(s_f,a)=0$, and $V^{\pi,\rho}_h(s_f)=0$.
\end{remark}
Now we show that the robust Q-function $Q_h^{\pi, \rho}(\cdot, \cdot)$ is linear in the feature mapping $\bphi(\cdot,\cdot)$ for any policy $\pi$. 
\begin{proposition}
\label{prop:Linear form}
    Under \Cref{assumption:linear MDP,assumption:fail-state}, for any $(\pi,s,a,h)\in \Pi\times\cS \times \cA \times [H]$, the robust Q-function $Q_h^{\pi,\rho}(s,a)$ has a linear form as follows: %
    {\small
    \begin{align*}
        {\textstyle Q_h^{\pi,\rho}(s,a) = \la \bphi(s,a), \btheta_h + \bnu_{h}^{\pi,\rho} \ra \ind\{s\neq s_f\}},
    \end{align*}}%
    where {\small$\bnu_h^{\pi,\rho}=(\nu_{h,1}^{\pi, \rho},\ldots,\nu_{h,d}^{\pi, \rho})^{\top}$}, {\small$\nu_{h,i}^{\pi, \rho}=\max_{\alpha\in[0,H]}\{ \allowbreak z_{h,i}^{\pi}(\alpha)-\rho\alpha\}$}, and {\small$z_{h,i}^{\pi}(\alpha)=\EE^{\mu_{h,i}^0}[V_{h+1}^{\pi, \rho}(s')]_{\alpha}$}.
\end{proposition}
Therefore, with the known feature mapping $\bphi(\cdot, \cdot)$, it suffices to estimate the weight vectors $\{\bnu_h^{\pi, \rho}\}_{h=1}^{H}$ to recover the robust Q-functions. 
Based on \Cref{prop:Linear form}, we can iteratively perform backward induction to estimate the robust Q-functions.
Specifically, %
given any estimated robust Q-function at step $h+1$, $Q_{h+1}^k(s,a)$, and estimated robust value function  $V_{h+1}^k(s)=\max_{a\in \mathcal{A}} Q_{h+1}^k(s, a)$, the one step backward induction leads to the following linear term
{\small
\begin{align*}
    {\textstyle \big\langle \bphi(s,a), \btheta_h+\bnu_h^{\rho,k}\big\rangle \ind\{s\neq s_f\}},
\end{align*}}%
where $\nu_{h,i}^{\rho} := \max_{\alpha \in [0,H]}\{z_{h,i}(\alpha)-\rho\alpha\}$ and $z_{h,i}(\alpha) := \EE^{\mu_{h,i}^0}[V_{h+1}^k(s') ]_{\alpha}$, for any $i\in[d]$. %
According to the linear structure defined in \Cref{assumption:linear MDP} on the nominal transition kernel, $z_{h,i}(\alpha)$ is the parameter of the following linear formulation,  
{\small
\begin{align*}
{\textstyle \big[\PP_h^0\big[V_{h+1}^k\big]_{\alpha}\big](s,a) =\big\la \bphi(s,a),\bz_h(\alpha) \big\ra},
\end{align*}}%
which is an expectation with respect to the nominal transition kernel $P_h^0$. Therefore, we can collect trajectories $\{(s_h^{\tau}, a_h^{\tau}, s_{h+1}^{\tau})\}_{\tau=1}^{k-1}$ and estimate $z_{h,i}(\alpha)$ from samples.
In particular, we will solve the following ridge regression problem with regularizer $\lambda>0$,
{\small
\begin{align}
\label{eq:ridge regression}
{\textstyle \hat{\bz}_h(\alpha)} &{\textstyle=\argmin_{\bz \in \RR^d}\sum_{\tau=1}^{k-1} \big(\big[V_{h+1}^k(s_{h+1}^{\tau})\big]_{\alpha}} \notag\\
& \qquad {\textstyle- \bphi_h^{\tau \top}\bz\big)^2 + \lambda\Vert \bz \Vert_2^2},
\end{align}}%
with the close-form solution being 
{\small
\begin{align}\label{eq:closed form}
    {\textstyle \hat{\bz}_h(\alpha)= \big(\Lambda_h^k \big)^{-1} \big[
                    \sum_{\tau=1}^{k-1}\bphi_h^\tau[V_{h+1}^k(s_{h+1})]_{\alpha}
                    \big]},
\end{align}}%
where $\bphi_h^\tau$ is a shorthand notation for $\bphi(s_h^{\tau}, a_h^{\tau})$, and $\Lambda_h^k=\sum_{\tau=1}^{k-1}\bphi_h^{\tau}(\bphi_h^{\tau})^{\top}+ \lambda\Ib$.
 We then approximate $\bnu_h^{\rho,k}$ by $\hat{\nu}_{h,i}^{\rho, k} = \max_{\alpha \in [0,H]}\{\hat{z}_{h,i} (\alpha)-\rho\alpha\}, i\in[d]$, and obtain the estimated robust Q-function at step $h$: 
{\small
\begin{align}\label{eq:ridge_est_Q}
{\textstyle Q_h^k(s,a) = \big\langle \bphi(s,a), \btheta_h+\hat{\bnu}_h^{\rho,k}\big\rangle \ind\{s\neq s_f\}}.
\end{align}}%
\begin{remark}\label{remark:linear regression} 
Thanks to the linear representation of the robust Q-function in terms of $\bphi(\cdot,\cdot)$ (\Cref{prop:Linear form}) and the linear dependence on the value function $V(s)$ in strong duality (\Cref{prop:strong duality for TV}), we can apply ridge regression with the estimated value function {\small$V_{h+1}^k(s')$} as the target. 

In comparison, the strong duality under KL uncertainty set \citep{shi2022distributionally} is  {\small$\inf_{\mu\in\cU^{\rho}(\mu^0)}\EE_{s\sim\mu}V(s) = \max_{\alpha\in[0, H/\rho]}\{-\alpha\log\EE_{s\sim\mu^0}[e^{-V(s)/\alpha}]-\alpha\rho\}$}.
Since the expectation is nonlinear in the value function $V(s)$, we have to apply ridge regression with $\exp(-V(s)/\alpha)$ as the target, and take logarithm back to the approximator (see (8) - (10) of \cite{ma2022distributionally} for details).
This logarithm operation could amplify the approximation error by $\exp(H)$, 
 which leads to the $O(\exp(H/\underline{\beta}))$ term in Theorem 4.1 of \cite{ma2022distributionally}. This amplified error could accumulate through the backward induction and ultimately lead to an $O(\exp(H^2))$ term in the regret bound of online DRMDPs.
Similar argument applies to the Chi-square divergence based uncertainty set, with the strong duality 
 {\small$ \inf_{\mu\in\cU^{\rho}(\mu^0)}\EE_{s\sim\mu}V(s) =\max_{\alpha\in[0, H]}\{\EE_{s\sim \mu^0}[V(s)]_{\alpha}-\sqrt{\rho\text{Var}_{s\sim \mu^0}([V(s)]_{\alpha})}\}$} \citep{shi2023curious}.
The non-linearity could lead to $O(H)$ error amplification in the regression approximation and $O(H^H)$ error accumulation in the regret bound in online DRMDPs. This justifies our choice of TV distance in the definition of the $d$-rectangular uncertainty set.

\end{remark}

\subsection{UCB Exploration in DRMDP}
In online DRMDPs, the ridge estimator in \eqref{eq:ridge_est_Q} is not sufficient for finding the optimal robust policy due to being greedy on past data that provides only partial information of the environment. Hence, we propose to incorporate a robust Upper Confidence Bonus (UCB) in the Q-function estimation to explore the source environment to avoid such myopic behavior. 

We present our algorithm DR-LSVI-UCB in \Cref{alg:DR-LSVI-UCB}. In each episode, DR-LSVI-UCB consists of two phases. In Phase 1 (Lines \ref{algline:start}-\ref{algline:end of the first loop}), it updates the robust Q-function estimation through backward induction. Specifically, the parameters used to form the robust Q-function estimation are updated by first solving ridge regressions according to \eqref{eq:ridge regression} and then solving optimization problems derived from \Cref{prop:strong duality for TV}. Next, 
a robust UCB is added to the Q-function estimation, whose exact form will be discussed in \Cref{remark:alg-8&9}. Finally, we truncate the robust Q-function at the fail state, by setting $Q_h^{k,\rho}(s_f,a)=0$ for any $a\in\cA$. In Phase 2 (Lines \ref{algline:begin of the second loop}-\ref{algline:end of the second loop}), it executes the greedy policy associated with the estimated robust Q-function to explore the source domain, and collects a new trajectory.

\begin{remark}
    \label{remark:alg-8&9} 
    In Line \ref{algline:update_nu} of \Cref{alg:DR-LSVI-UCB}, we denote {\small$\alpha_i^{\star} = \argmax_{\alpha\in[0,H]}\{z_{h,i}^k(\alpha)-\rho\alpha\}$} for any $i\in[d]$. Then we compute 
    {\small$\nu_{h,i}^{\rho,k}=z_{h,i}^k(\alpha_i^{\star})-\rho\alpha_i^{\star}$}, where {\small$z_{h,i}^k(\alpha_i^{\star})$} is the $i$-th element of vector $\bz_{h}^k(\alpha_i^{\star})$. This immediately implies that we have to solve $d$ distinct ridge regressions in Line \ref{algline:update_z} to obtain different coordinates of $\bnu_h^{\rho,k}$. This further leads to our design of the robust UCB term in Line \ref{algline:UCB},  {\small$\beta\sum_{i=1}^d\phi_i(s, a)[\mathbf{1}_i^\top(\Lambda_h^k)^{-1}\mathbf{1}_i]^{1/2}$}, which consists of $d$ different upper confidence bonuses. 
    This design is motivated from the optimism principle used in standard MDPs \citep{azar2017minimax, jin2020provably}, where a bonus term proportional to the approximation error is added to guide exploration.
    A distinctive feature of the robust UCB term in \Cref{alg:DR-LSVI-UCB} is that the approximation error arises from $d$ ridge regressions, due to the $d$-rectangular uncertainty set.

\end{remark}

\begin{remark}
\label{remark:heterogeneous-rho}
In practice, \Cref{alg:DR-LSVI-UCB} extends to broader scenarios, where uncertainty level $\rho$ varies across different uncertainty sets {\small$\{\cU_{h,i}^{\rho}(\mu_{h,i}^0)\}_{i, h=1}^{d,H}$}. We denote {\small$\bm{\rho}=\{\rho_{h,i}\}_{i, h=1}^{d,H}$}, where $\rho_{h,i}$ is the uncertainty level for the $i$-th factor uncertainty set at step $h$. To generalize \Cref{alg:DR-LSVI-UCB}, we simply replace $\rho$ in Line \ref{algline:update_nu} with $\rho_{h,i}$. This updated algorithm handles varied uncertainty levels with $\bm{\rho}$ chosen to satisfy various objectives. Importantly, due to the bounded range of \intextmath{$\{\rho_{h,i}\}_{i, h=1}^{d,H}$} in $[0,1]$ and the independence of factor uncertainty sets, heterogeneity in uncertainty level does not impact our analysis. Therefore, the modified algorithm maintains the average suboptimality bound of the original algorithm, as depicted in \Cref{sec:Algorithm and theoretical analysis}.
\end{remark}

\begin{algorithm}[ht]
    \caption{DR-LSVI-UCB}\label{alg:DR-LSVI-UCB}
    \begin{algorithmic}[1]
    \REQUIRE{
        Parameters $\beta>0$ and $\lambda>0$
        }
        \FOR {episode $k=1, \cdots, K$}{
            \STATE Receive the initial state $s_1^k$. \label{algline:start}
            \FOR {stage $h=H, \cdots, 1$}{
                \STATE {\small$\Lambda_h^k \leftarrow \sum_{\tau=1}^{k-1}\bphi(s_h^{\tau}, a_h^{\tau})\bphi(s_h^{\tau}, a_h^{\tau})^\top + \lambda \bI$}
                \IF{$h=H$}{
                    \STATE $\bnu_h^{\rho, k} \leftarrow 0$
                }
                \ELSE{
                    \STATE Update $z_h^k(\alpha)$ according to \eqref{eq:closed form}.
                    \label{algline:update_z}

                    \STATE $\nu_{h,i}^{\rho, k} \leftarrow \max_{\alpha\in [0,H]}\{z_{h,i}^k(\alpha)-\rho\alpha\}, i \in [d]$ \label{algline:update_nu}
                }
                \ENDIF
                \STATE {\small$\Gamma_{h}^k(s, a) \leftarrow \beta\sum_{i=1}^d\phi_i(s, a)[\mathbf{1}_i^\top(\Lambda_h^k)^{-1}\mathbf{1}_i]^{1/2}$ }
                \label{algline:UCB}
                \STATE {\small$Q_{h}^{k, \rho}(s, a) \leftarrow \min \{\bphi(s, a)^\top(\btheta_h+\bnu_h^{\rho, k}) + \Gamma_{h}^k(s, a), H-h+1 \}_+\ind\{s\neq s_f\}$}
                \label{algline:update Q}
                \STATE {\small$\pi_{h}^{k} (s)\leftarrow \argmax_{a\in \mathcal{A}}Q_{h}^{k, \rho}(s, a)$}
                \label{algline:update policy}
            }
            \ENDFOR \label{algline:end of the first loop}
            \FOR{stage $h=1, \cdots, H$}\label{algline:begin of the second loop}
            {
                \STATE Take the action $a_h^k \leftarrow \pi_{h}^{k}(s_h^k)$, and receive the next state $s_{h+1}^k$.
            }
            \ENDFOR \label{algline:end of the second loop}
        }
        \ENDFOR
    \end{algorithmic}
\end{algorithm}

\section{MAIN THEORETICAL RESULTS}
\label{sec:Algorithm and theoretical analysis}

Now we present our main result for \Cref{alg:DR-LSVI-UCB}.%
\begin{theorem}
\label{th:DRLSVIUCB}
Under \Cref{assumption:linear MDP,assumption:fail-state}, there exists an absolute constant $c > 0$ such that, for any fixed $p \in (0,1)$, if we set $\lambda=1$ and $\beta = c\cdot dH\sqrt{\iota}$ with $\iota=\log(3dKH/p)$ in \Cref{alg:DR-LSVI-UCB}, then with probability at least $1-p$ the average suboptimality of DR-LSVI-UCB satisfies
{\small
\begin{align}
\label{eq:AveSubopt-th}
   &{\textstyle \text{AveSubopt}(K) \leq \sqrt{{2H^3\log(3/p)}/{K}}} \notag \\
    & \qquad +{2\beta}/{K}\underbrace{{\textstyle\sum_{k=1}^K\sum_{h=1}^H\sum_{i=1}^d\phi_{h,i}^k\sqrt{\mathbf{1}_i^\top (\Lambda_h^k)^{-1}\mathbf{1}_i}}}_{d\text{-rectangular estimation error}},
\end{align}}%
where $\phi_{h,i}^k$ is the $i$-th element of $\bphi_{h}^k = \bphi(s_h^k, a_h^k)$ and $\mathbf{1}_i$ is the one hot vector with its $i$-th entry being 1.
\end{theorem}

The %
$d\text{-rectangular estimation error}$ in \eqref{eq:AveSubopt-th}
resembles the regression error  
{\small $\sum_{k=1}^K\sum_{h=1}^H\sqrt{(\bphi_h^k)^{\top}(\Lambda_h^k)^{-1}\bphi_h^k}$} in the standard episodic linear MDP literature \citep{jin2020provably, he2021logarithmic, he2023nearly}. However, it cannot be easily bounded by the elliptical potential lemma \cite[Lemma 11]{abbasi2011improved}, as its summands are not quadratic terms $\Vert\phi_h^k\Vert_{(\Lambda_h^k)^{-1}}$ but weighted sum of diagonal elements of $(\Lambda_{h}^k)^{-1}$, i.e.,  {\small $\sum_{i=1}^d\phi_{h,i}^k[(\Lambda_{h}^k)^{-1}]_{ii}^{1/2}$}. 
As shown in \Cref{remark:alg-8&9}, this term primarily originates from the necessity to solve $d$ distinct ridge regressions at each episode $k$ and step $h$, due to the structure of the $d$-rectangular uncertainty set. 
This represents a unique challenge in DRMDPs analysis with linear function approximation. Similar terms also appear in the proof of Theorem 4.1 in \citet{ma2022distributionally} and Theorem 6.3 in \citet{blanchet2023double}, which share our setting. However, their final results do not explicitly showcase this due to strong coverage assumptions on offline dataset, which may not hold in practice and are inapplicable to the off-dynamics learning setting in our paper, which requires active and incremental data collection via interaction with the source environment.

In the following, we will instantiate the average suboptimality bound in \Cref{th:DRLSVIUCB} on different examples.  
We start with the tabular MDP, where the number of states and actions are finite. We set dimension $d=|\cS|\times|\cA|$ and the feature mapping $\bphi(s,a)=\be_{(s,a)}$ as the canonical basis in $\RR^d$. Then the $d$-rectangular assumption degenerates to the $(s,a)$-rectangular assumption \citep{goyal2023robust}.
It turns out that with this specific structure of feature mapping $\bphi(s,a)=\be_{(s,a)}$, we can bound the $d$-rectangular estimation error without further assumption. %

\begin{corollary}
\label{corollary:DRLSVIUCB-tabular}
    Under the setting of tabular MDP with $|\cS| = S$ and $|\cA|=A$, there exists an absolute constant $c > 0$ such that, for any fixed $p \in (0,1)$, if we set $\lambda$ and $\beta$ in \Cref{alg:DR-LSVI-UCB} as in \Cref{th:DRLSVIUCB}, then with probability at least $1-p$, the average suboptimality of DR-LSVI-UCB is at most \intextmath{$\tilde{\cO}(\sqrt{H^4S^3A^3/K)}$}.
\end{corollary}
Note that $d=SA$ in the tabular setting. Our result in \Cref{corollary:DRLSVIUCB-tabular} aligns with the average regret bound \intextmath{$\tilde{O}(\sqrt{H^4d^3/K})$} of LSVI-UCB in standard MDP, which can be derived by dividing the cumulative regret bound in Theorem 3.1 of \cite{jin2020provably} by $K$. 
In addition, \citet{dong2022online} also studied the online DRMDP problem under the $(s,a)$-rectangular assumption and proposed an algorithm with an average suboptimality bound of \intextmath{$\tilde{O}(\sqrt{H^4S^2A/K})$}, improving our result by a factor of $\sqrt{S}A$. However, their algorithm is model-based and only designed for $(s,a)$-rectangular robust tabular MDPs, which is not extendable to the function approximation setting. In contrast, our DR-LSVI-UCB algorithm is model-free and amenable to function approximation. Moreover, \algname\ is designed for the more general $d$-rectangular linear DRMDPs, covering a broader scope than solely the $(s,a)$-rectangular robust tabular MDPs.

Next, we consider the general $d$-rectangular linear DRMDP setting. Under an assumption on the inherent structure of linear MDP, we have the following average suboptimality bound.
\begin{corollary}
\label{corollary:DRLSVIUCB}
    For all $(\pi, h) \in \Pi \times [H]$, assume that
\begin{align}
    \label{eq:Feature exploration}
    {\textstyle \EE_{\pi}[\bphi(s_h,a_h)\bphi(s_h, a_h)^\top ] \geq \alpha I},
\end{align}
where $\alpha>0$. Then there exists an absolute constant $c > 0$ such that, for any fixed $p \in (0,1)$, if we set $\lambda$ and $\beta$ in \Cref{alg:DR-LSVI-UCB} as in \Cref{th:DRLSVIUCB}, then with probability at least $1-p$ the average suboptimality of DR-LSVI-UCB is at most \intextmath{$\tilde{\cO}(\sqrt{d^2H^4/(\alpha^2K)})$}.
\end{corollary}

\begin{remark}\label{rmk:compare_with_data_coverage}
    Note that $\alpha$ represents the lower bound of the smallest eigenvalue of $\EE_{\pi}[\bphi(s_h,a_h)\bphi(s_h, a_h)^\top ]$, which can be upper bounded by $1/d$ \citep{wang2020statistical}. When $\alpha=O(1/d)$, \Cref{corollary:DRLSVIUCB} suggests an average suboptimality bound of \intextmath{$\tilde{\cO}(\sqrt{d^4H^4/K})$}. 
    Moreover, \cite{blanchet2023double} studied the offline setting of $d$-rectangular linear DRMDP with TV uncertainty set. Under the robust partial coverage assumption on the offline dataset, their model-based algorithm P$^2$MPO achieves \intextmath{$\tilde{O}(\sqrt{d^4H^4/c^{\dagger}K})$} suboptimality bound, where $c^\dagger$ is a problem dependent %
    constant related to the robust partial coverage assumption. If we further assume $c^{\dagger}=O(1)$, then the suboptimality bound of P$^2$MPO is the same as DR-LSVI-UCB. 
    
    In contrast with P$^2$MPO, DR-LSVI-UCB does not require a precollected offline dataset satisfying the strong coverage assumption, which is unrealistic in practice. In particular, the robust partial coverage assumption requires that the offline dataset has sufficient coverage of distributions induced by the optimal robust policy and any transition kernel in the uncertainty set. Since the optimal robust policy is unknown, and there are infinite transition kernels in the uncertainty set, it's practically impossible to verify this robust partial coverage assumption. Instead, our algorithm employs an online incremental approach to explore data through active interactions with the source domain.
    Additionally, we can numerically compute the $d\text{-rectangular estimation error}$ in  \eqref{eq:AveSubopt-th}, and then acquire a specific value of the high probability upper bound of the average suboptimality according to \eqref{eq:AveSubopt-th}.
    
    In addition, P$^2$MPO is computationally intractable. For example, even when the model space in their algorithm is specified for $d$-rectangular linear DRMDPs, their algorithm requires exact solution of a supremum problem, $\sup_{\nu\in\cV}$, over the value function class to obtain a confidence region $\widehat{\cP}_h$, and the solution of an infimum problem, $\inf_{P_h\in \widehat{\cP}_h}$, over the confidence region $\widehat{\cP}_h$ (see (3.1) and (6.1) in \citet{blanchet2023double} for details). These requirements make P$^2$MPO computationally intractable. In contrast, our proposed DR-LSVI-UCB algorithm is not only statistically efficient, but also computationally efficient. %
    
\end{remark}

\begin{remark}
    When $\alpha=O(1/d)$, the average suboptimality bound of DR-LSVI-UCB, \intextmath{$\tilde{O}(\sqrt{d^4H^4/K})$}, matches the average regret bound \intextmath{$\tilde{O}(\sqrt{d^3H^4/K})$} for LSVI-UCB in standard linear MDPs \citep[Theorem 3.1]{jin2020provably} with respect to horizon length $H$ and number of episodes $K$. However, our result in the robust setting incurs an extra $\sqrt{d}$ term concerning the feature dimension. This factor emerges from the necessity for \Cref{alg:DR-LSVI-UCB} to solve $d$ distinct ridge regressions to estimate the parameter of the $d$-rectangular uncertainty set (refer to Lines \ref{algline:update_z}, \ref{algline:update_nu}, \ref{algline:update Q} of \Cref{alg:DR-LSVI-UCB}). 
    An intriguing open question remains whether this additional $\sqrt{d}$ factor can be mitigated through algorithm design or a more refined analysis.

\end{remark}

\begin{figure*}[t]
    \centering
    \subfigure[$\Vert\xi\Vert_1 = 0.1$, $\rho_{1,4}=0.5$]{\includegraphics[scale=0.35]{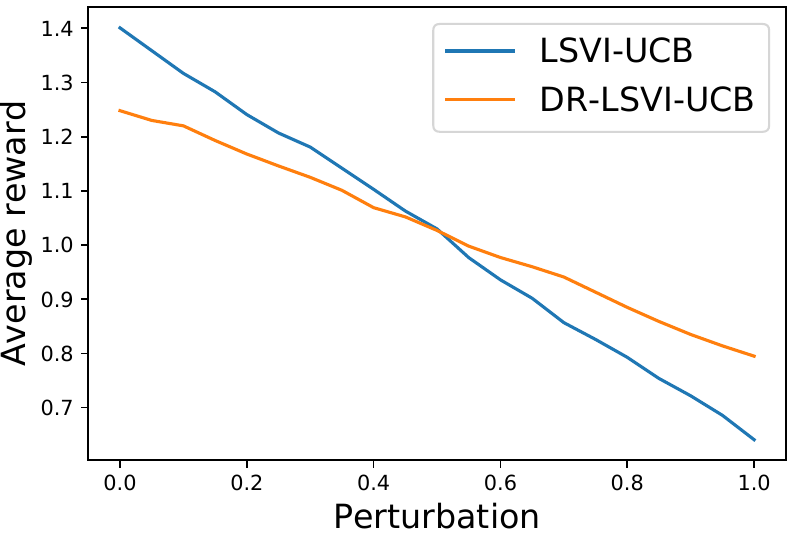}} 
     \subfigure[$\Vert\xi\Vert_1 = 0.2$, $\rho_{1,4}=0.5$]{\includegraphics[scale=0.35]{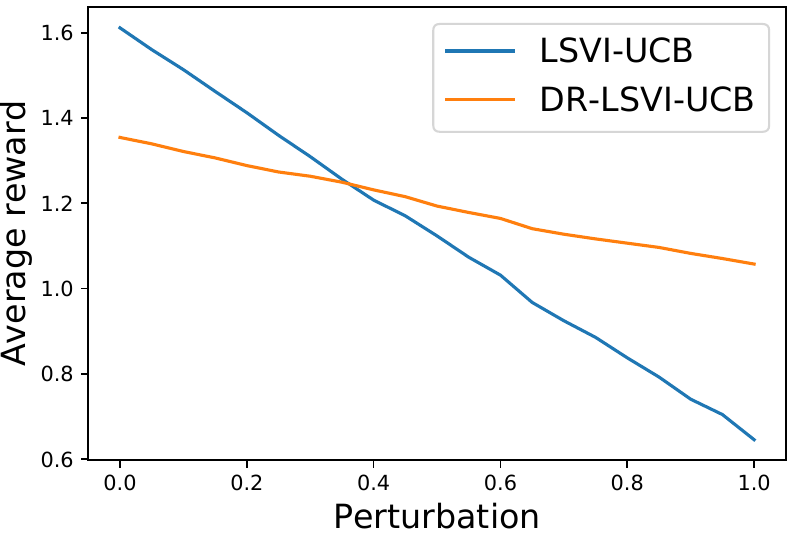}}
      \subfigure[$\Vert\xi\Vert_1 = 0.3$, $\rho_{1,4}=0.5$]{\includegraphics[scale=0.35]{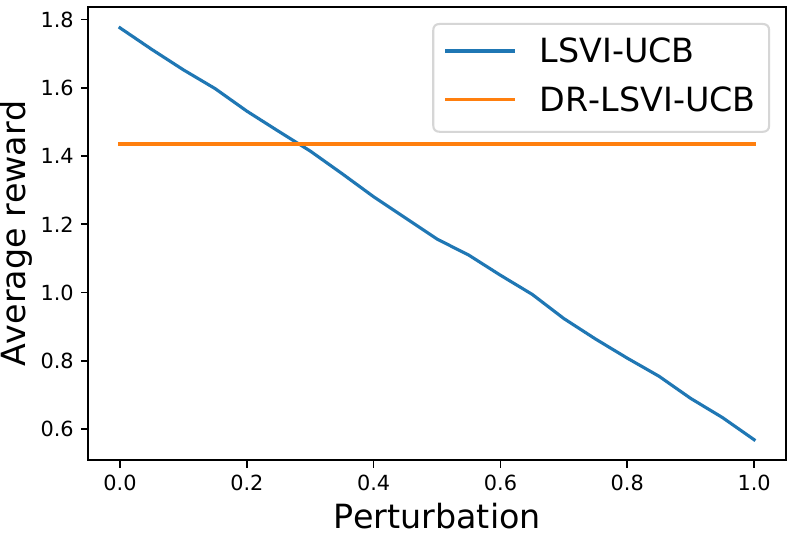}}
    \caption{Simulation results under different source domains. The $x$-axis represents the perturbation level corresponding to different target environments. $\rho_{1,4}$ is the uncertainty level in our DR-LSVI-UCB algorithm.}
    \label{fig:simulation results}
\end{figure*}

\begin{figure*}[t]
    \centering
    \subfigure[$\rho=0.3$]{\includegraphics[scale=0.35]{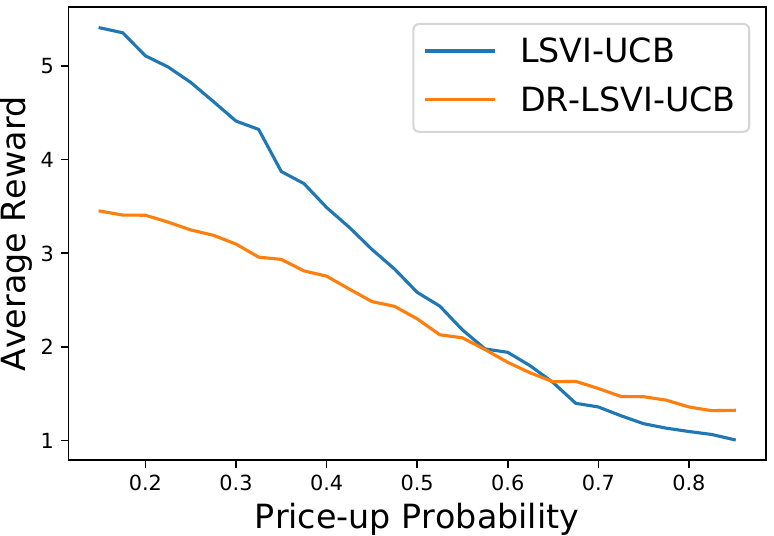}} 
     \subfigure[$\rho=0.4$]{\includegraphics[scale=0.35]{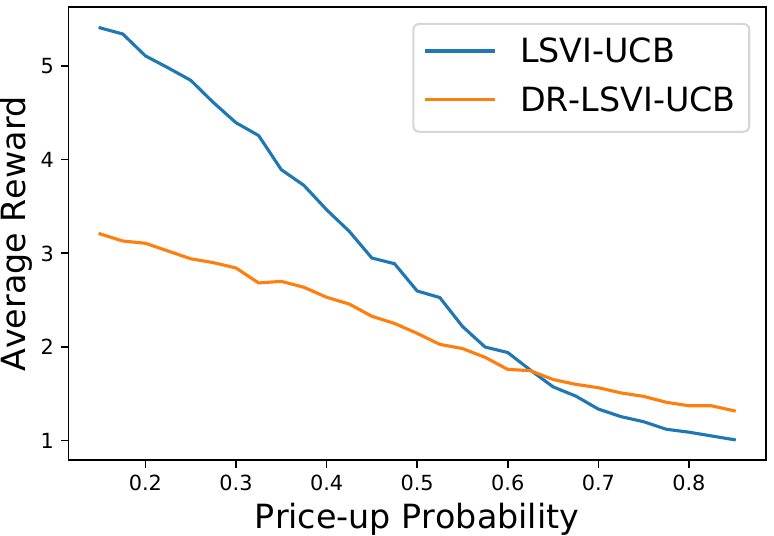}}
      \subfigure[$\rho=0.5$]{\includegraphics[scale=0.35]{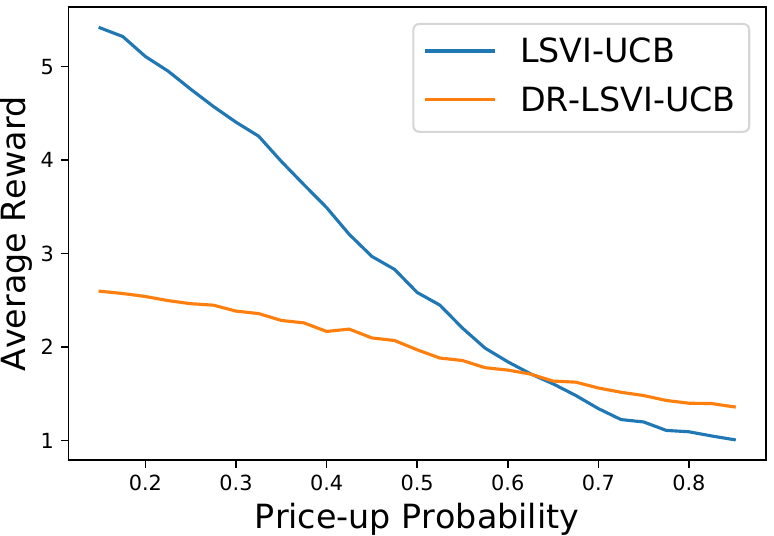}}
    \caption{Results for the simulated American put option problem. $\rho$ is the uncertainty level in DR-LSVI-UCB.}
    \label{fig:simulation results-APO}
\end{figure*}

\section{EXPERIMENTS}
\label{sec:Experiments}
In this section, we compare DR-LSVI-UCB with its non-robust counterpart, LSVI-UCB \citep{jin2020provably}, on two off-dynamics RL problems. 
All numerical experiments were conducted on a MacBook Pro with a 2.6 GHz 6-Core Intel CPU. The implementation of our \algname\ algorithm is available at \url{https://github.com/panxulab/Distributionally-Robust-LSVI-UCB}.

\subsection{Simulated Off-Dynamics Linear MDPs}\label{sec:experiments_simulatedMDP}
 We first construct a linear MDP as the source domain, where the learning horizon $H=3$, and the state space is $\cS=\{x_1, \cdots, x_5\}$. At each step, the action $a$ is chosen from $\cA=\{-1,1\}^4\subset\RR^4$. The initial state is always $x_1$, which can transit to $x_2$, $x_4$ or $x_5$ with nonzero probabilities, where $x_4$ and $x_5$ are absorbing states. From $x_2$, the next state can be $x_3$, $x_4$ or $x_5$, and from $x_3$, it can only transit to $x_4$ or $x_5$. 
 We design the transition probabilities and rewards such that they both depend on $\la \xi, a\ra$, which is bounded in $[-\Vert \xi \Vert_1, \Vert \xi \Vert_1]$ by the definition of $\cA$, where $\xi\in\RR^4$ is a hyperparameter of the MDP instance. We verify that this MDP satisfies \Cref{assumption:linear MDP} with $d=4$. 
 We then construct target domains by perturbing the transition probability at $x_1$ of the source domain such that the divergence is up to $q\in(0,1)$ in TV distance.
 Due to the space limit, we defer more details on the construction and verification of the source domain as well as the perturbation of the target domain to \Cref{sec:supp-SLMDP}.
 
 In our experiments, we consider different source MDP instances by setting $\|\xi\|_1\in\{0.1, 0.2, 0.3\}$. To implement the uncertainty set in \algname, we use heterogeneous uncertain levels $\rho_{h,i}$ for $h\in[H]$ and $i\in[d]$ as we discussed in \Cref{remark:heterogeneous-rho}. 
 In particular, we set $\rho_{1,4} = 0.5$ and $\rho_{h,i}=0$ for all other cases. We evaluate different policies based on their average rewards achieved in the target domain, which are illustrated in \Cref{fig:simulation results}. It can be seen that LSVI-UCB outperforms \algname\ when the dynamics shift is small, but significantly underperforms when the dynamics shift is moderate or substantial, which verifies the robustness of our \algname. We also conduct an ablation study on the effect of different values of $\rho_{1,4}$ on the performance of \algname, which is deferred to \Cref{sec:supp-SLMDP} due to the space limit.

\subsection{Simulated American Put Option}
We then evaluate our algorithm in a simulated American put option problem \citep{tamar2014scaling, zhou2021finite, ma2022distributionally}. 
There is a price model in this problem, which is assumed to follow the Bernoulli distribution
\begin{align}
    s_{h+1}=
    \begin{cases}
     1.02 s_h, &\text{w.p.} ~ p_u\\
     0.98 s_h, &\text{w.p.} ~ 1-p_u, \label{eq:bernoulli}
    \end{cases}
\end{align}
where $p_u\in(0,1)$ is the probability that the price goes up in the next step. The initial price $s_0$ is generated uniformly from $[95, 105]$. At each step $h$, an agent can take one of the two actions: exercising the option ($a_e$) or not exercising the option ($a_{ne}$). If exercising the option, the agent receives a reward of $\max\{0, 100-s_h\}$, and the next state would be the exit state. If not exercising the option, the agent receives $0$ reward, and the next state $s_{h+1}$ is generated based on the Bernoulli distribution in \eqref{eq:bernoulli}. We limit the number of trading steps to $H$. 

In order to employ linear function approximation, we construct a feature mapping $\phi:
\cS \times \cA \rightarrow \RR^{d+1}$ motivated by \cite{ma2022distributionally}. Specifically, we first construct the set of anchor states, $\{s_i\}_{i=1}^d$, where $s_1=80$, $s_{i+1}-s_i=\Delta$ and $\Delta = 60/d$. Then we define,
\begin{align*}
    \phi(s_h, a) = 
    \begin{cases}
        [\varphi_1(s_h),\cdots, \varphi_d(s_h), 0], &\text{if}~a=a_e \\
        [0, \cdots, 0, \max\{0, 100-s_h\}], &\text{if}~a=a_{ne},
    \end{cases}
\end{align*}
where $\varphi_i(s) = \max\{0, 1-|s_h-s_i|/\Delta\},~i\in[d]$. 
In our simulation, we set the price-up probability of the source domain to $p_u=0.5$, maximum trading steps $H$ to 10, and the feature dimension $d$ to 20. Moreover, we consider various target domains, each with a price-up probability falling within the range of $[0.15,0.85]$.
We conduct experiments on different uncertainty levels $\rho$ for \algname, and plot the average rewards for LSVI-UCB and \algname\ on target domains in \Cref{fig:simulation results-APO}.
It can be seen that the average rewards of robust policies are more stable over different target domains. In particular, \algname\ outperforms LSVI-UCB under worst-cases when the price-up probability of the target domain is much higher than that of the source domain.
 
\section{CONCLUSION}
We studied off-dynamics RL under the framework of online DRMDPs with linear function approximation. We proposed a model-free algorithm \algname, which learns the optimal robust policy through active interaction with the source domain. This is the first provably efficient DRMDP algorithm for off-dynamics RL with function approximation. We established the first non-asymptotic suboptimality bound for this setting, which is independent of state and action space sizes. We  validated the performance and robustness of DR-LSVI-UCB on carefully designed instances.
It remains an intriguing open question whether the theoretical bounds for online DRMDPs can match that of standard linear MDPs. It is also of great interest to derive lower bounds on $d$-rectangular linear DRMDPs to see its fundamental limits.

\section*{Acknowledgements}
We would like to thank the anonymous reviewers for their helpful comments. PX was supported in part by the National Science Foundation (DMS-2323112) and the Whitehead Scholars Program at the Duke University School of Medicine. The views and conclusions in this paper are those of the authors and should not be interpreted as representing any funding agencies.

\bibliography{reference.bib}

 \begin{enumerate}

 \item For all models and algorithms presented, check if you include:
 \begin{enumerate}
   \item A clear description of the mathematical setting, assumptions, algorithm, and/or model. [Yes]
   \item An analysis of the properties and complexity (time, space, sample size) of any algorithm. [Yes]
   \item (Optional) Anonymized source code, with specification of all dependencies, including external libraries. [Yes]
 \end{enumerate}

 \item For any theoretical claim, check if you include:
 \begin{enumerate}
   \item Statements of the full set of assumptions of all theoretical results. [Yes]
   \item Complete proofs of all theoretical results. [Yes]
   \item Clear explanations of any assumptions. [Yes]     
 \end{enumerate}

 \item For all figures and tables that present empirical results, check if you include:
 \begin{enumerate}
   \item The code, data, and instructions needed to reproduce the main experimental results (either in the supplemental material or as a URL). [Yes]
   
   The code of our implementation is available at \url{https://github.com/panxulab/Distributionally-Robust-LSVI-UCB}.
   \item All the training details (e.g., data splits, hyperparameters, how they were chosen). [Yes]
         \item A clear definition of the specific measure or statistics and error bars (e.g., with respect to the random seed after running experiments multiple times). [Yes]
         \item A description of the computing infrastructure used. (e.g., type of GPUs, internal cluster, or cloud provider). [Yes]
 \end{enumerate}

 \item If you are using existing assets (e.g., code, data, models) or curating/releasing new assets, check if you include:
 \begin{enumerate}
   \item Citations of the creator If your work uses existing assets. [Not Applicable]
   \item The license information of the assets, if applicable. [Not Applicable]
   \item New assets either in the supplemental material or as a URL, if applicable. [Not Applicable]
   \item Information about consent from data providers/curators. [Not Applicable]
   \item Discussion of sensible content if applicable, e.g., personally identifiable information or offensive content. [Not Applicable]
 \end{enumerate}

 \item If you used crowdsourcing or conducted research with human subjects, check if you include:
 \begin{enumerate}
   \item The full text of instructions given to participants and screenshots. [Not Applicable]
   \item Descriptions of potential participant risks, with links to Institutional Review Board (IRB) approvals if applicable. [Not Applicable]
   \item The estimated hourly wage paid to participants and the total amount spent on participant compensation. [Not Applicable]
 \end{enumerate}

 \end{enumerate}

\newpage
\appendix
\onecolumn

\section{EXPERIMENT SETUP AND ADDITIONAL RESULTS}
\label{sec:Experiment setup and additional results}
In this section, we provide additional details and more experimental results for our numerical study in \Cref{sec:Experiments}. 

\subsection{Simulated Linear MDP}\label{sec:supp-SLMDP}
We first describe the details about the construction of the source and target linear MDPs in \Cref{sec:experiments_simulatedMDP} and then provide the implementation of our method. We also present more ablation study on the robustness of our method with respect to the input parameter $\rho_{1,4}$ which stands for the uncertainty level.

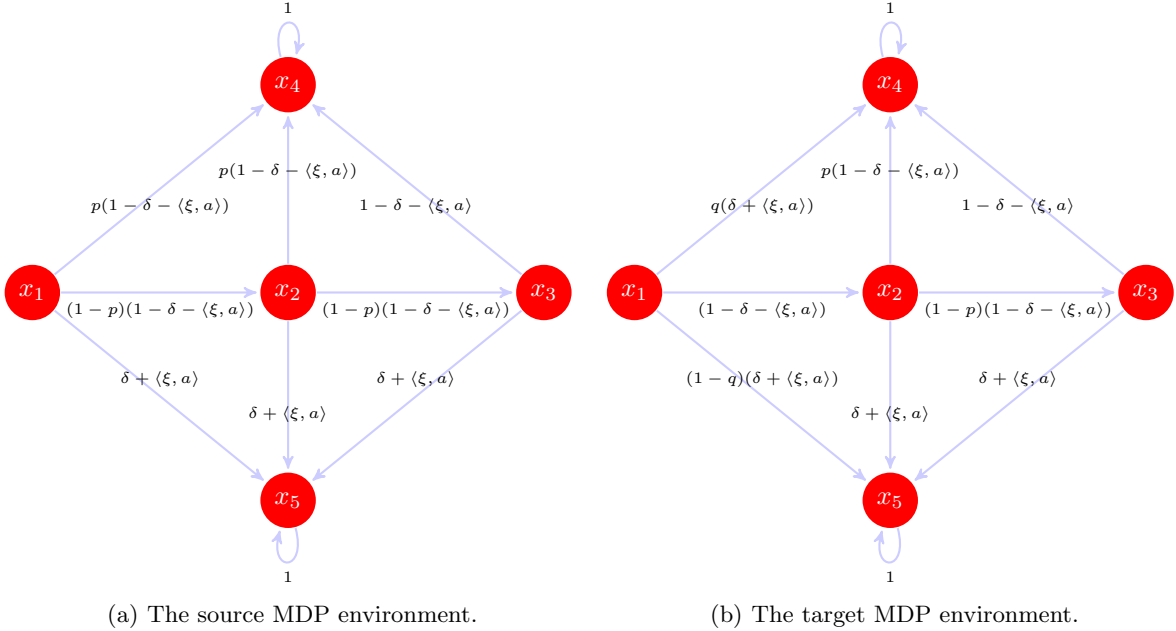
\begin{figure*}[ht]
    \centering
    \subfigure[The source MDP environment.]{
        \begin{tikzpicture}[->,>=stealth',shorten >=1pt,auto,node distance=3.4cm,thick]
            \tikzstyle{every state}=[fill=red,draw=none,text=white,minimum size=0.5cm]
            \node[state] (S1) {$x_1$};
            \node[state] (S2) [right of=S1] {$x_2$};
            \node[state] (S3) [right of=S2] {$x_3$};
            \node[state] (S4) [above=2cm of S2] {$x_4$};
            \node[state] (S5) [below=2cm of S2] {$x_5$};
            
            \path   (S1) edge[draw=blue!20] node[below] {\tiny $(1-p)(1-\delta-\la\xi,a\ra)$} (S2)
                         edge[draw=blue!20] node[below] {\tiny$p(1-\delta-\la\xi,a\ra)$} (S4)
                         edge[draw=blue!20] node[above] {\tiny$\delta+\la\xi,a\ra$} (S5)
                    (S2) edge[draw=blue!20] node[below] {\tiny$(1-p)(1-\delta-\la\xi,a\ra)$} (S3)
                         edge[draw=blue!20] node[above] {\tiny$p(1-\delta-\la\xi,a\ra)$} (S4)
                         edge[draw=blue!20] node[below] {\tiny$\delta+\la\xi,a\ra$} (S5)
                    (S3) edge[draw=blue!20] node[below] {\tiny$1-\delta-\la\xi,a\ra$} (S4)
                         edge[draw=blue!20] node[above] {\tiny$\delta+\la\xi,a\ra$} (S5)
                    (S4) edge[draw=blue!20] [loop above] node {\tiny 1} (S4)
                    (S5) edge[draw=blue!20] [loop below] node {\tiny 1} (S5);
        \end{tikzpicture}
        \label{fig:mdp_5states}
    }
    \subfigure[The target MDP environment.]{
        \begin{tikzpicture}[->,>=stealth',shorten >=1pt,auto,node distance=3.4cm,thick]
            \tikzstyle{every state}=[fill=red,draw=none,text=white,minimum size=0.5cm]
            \node[state] (S1) {$x_1$};
            \node[state] (S2) [right of=S1] {$x_2$};
            \node[state] (S3) [right of=S2] {$x_3$};
            \node[state] (S4) [above=2cm of S2] {$x_4$};
            \node[state] (S5) [below=2cm of S2] {$x_5$};
            
            \path   (S1) edge[draw=blue!20] node[below] {\tiny$(1-\delta-\la\xi,a\ra)$} (S2)
                         edge[draw=blue!20] node[below] {\tiny$q(\delta+\la\xi,a\ra)$} (S4)
                         edge[draw=blue!20] node[above] {\tiny$(1-q)(\delta+\la\xi,a\ra)$} (S5)
                    (S2) edge[draw=blue!20] node[below] {\tiny$(1-p)(1-\delta-\la\xi,a\ra)$} (S3)
                         edge[draw=blue!20] node[above] {\tiny$p(1-\delta-\la\xi,a\ra)$} (S4)
                         edge[draw=blue!20] node[below] {\tiny$\delta+\la\xi,a\ra$} (S5)
                    (S3) edge[draw=blue!20] node[below] {\tiny$1-\delta-\la\xi,a\ra$} (S4)
                         edge[draw=blue!20] node[above] {\tiny$\delta+\la\xi,a\ra$} (S5)
                    (S4) edge[draw=blue!20] [loop above] node {\tiny 1} (S4)
                    (S5) edge[draw=blue!20] [loop below] node {\tiny 1} (S5);
        \end{tikzpicture}
        \label{fig:perturbed_mdp}
    }
    \caption{The source and the target linear MDP environments. The value on each arrow represents the transition probability. For the source MDP, there are five states and three steps, with the initial state being  $x_1$, the fail state being $x_4$, and $x_5$ being an absorbing state with reward 1. The target MDP on the right is obtained by perturbing the transition probability at the first step of the source MDP, with others remaining the same. }
\end{figure*}

\paragraph{Construction of the linear MDP} 
The source environment MDP is showed in \Cref{fig:mdp_5states}.  We recall that the learning horizon is $H=3$, the state space is $\cS=\{x_i\}_{i=1}^5$, and the action space is $\cA=\{-1,1\}^4 \subset \RR^4$. The initial state in each episode is always $x_1$. We construct the feature mapping $\phi:\cS\times\cA \rightarrow \RR^d$ with $d=4$ as follows:
\begin{align*}
    \phi(x_1, a) &= (1-\delta-\la\xi, a \ra, 0, 0, \delta+\la\xi, a \ra)^{\top},\\
    \phi(x_2, a) &= (0, 1-\delta-\la\xi,a\ra, 0, \delta+\la\xi, a \ra)^{\top}, \\
    \phi(x_3, a) &= (0, 0, 1-\delta-\la\xi,a\ra, \delta+\la\xi, a \ra)^{\top}, \\
    \phi(x_4, a) &= (0, 0, 1, 0)^{\top},\\
    \phi(x_5, a) &= (0, 0, 0, 1)^{\top},
\end{align*}
where the $\delta$ and $\xi$ are hyperparameters. We then define the reward parameters $\btheta = \{\btheta_h\}_{h=1}^3$ as 
$$\btheta_1 = (0,0,0,0)^{\top},~ \btheta_2 = (0, 0, 0, 1)^{\top}~ \text{and}~ \btheta_3 = (0, 0, 0, 1)^{\top},$$ 
and the factor distributions $\{\bmu_h\}_{h=1}^2$ as 
\begin{align}\label{exp:def_factor_distribution_source}
 \bmu_1=\bmu_2=((1-p)\delta_{x_2}+p\delta_{x_4}, (1-p)\delta_{x_3}+p\delta_{x_4}, \delta_{x_4}, \delta_{x_5})^{\top},   
\end{align}
where the $\delta_x$ is a Dirac measure which puts an atom on element $x$, and $p$ is a hyperparameter.
With these notations, we define the linear reward functions as 
$$r_h(s,a)=\bphi(s,a)^{\top}\btheta_h, ~\forall (h,s,a)\in[H]\times\cS\times\cA,$$
and the linear transition kernels as $$P_h(\cdot|s,a)=\bphi(s,a)^{\top}\bmu_h(\cdot), ~\forall (h,s,a)\in[H]\times\cS\times\cA.$$
Note that by construction, $x_4$ is a fail state in this MDP as (i) $P_h(x_4|x_4,a)=1, \forall (h,a)\in [H]\times \cA$, and (ii) $r_h(x_4,a)=0,~ \forall (h, a)\in [H]\times\cA$.
Thus, it is easy to verify that the constructed source MDP satisfies \Cref{assumption:linear MDP,assumption:fail-state}. 
In our simulation, we set $p=0.001$, $\delta=0.3$, $\xi=(1/\Vert\xi\Vert_1, 1/\Vert\xi\Vert_1,1/\Vert\xi\Vert_1,1/\Vert\xi\Vert_1)^{\top}$ and $\Vert\xi\Vert_1 = \{0.1,0.2,0.3\}$.
Next, we construct several target domains, as showed in \Cref{fig:perturbed_mdp}, by perturbing the source domain. Specifically, we only perturb the factor distributions $\bmu_1$ in \eqref{exp:def_factor_distribution_source} for the fist step of the MDP, which is changed to
\begin{align}
    \bmu_1^{\text{perturbed}} =\big (\delta_{x_2}, \delta_{x_3}, \delta_{x_4}, (1-q)\delta_{x_5}+q\delta_{x_4}\big)^{\top},
\end{align}
where $q$ is a factor that controls the perturbation level. In our simulation, we consider difference values of $q$ in the range $[0, 1]$. Moreover, We train policies in the source domain through 100 epochs, and test those policies by computing the average reward in target domains through 100 epochs. 

\paragraph{Ablation study} We also conduct additional experiments to study the impact of $\rho_{1,4}$ on the robustness of our algorithm. In particular, we vary the value of $\rho_{1,4}$ in the range $\{0.3, 0.4, 0.5\}$ and set all other $\rho_{h,i}=0$.
Results of ablation study are showed in \Cref{fig:simulation-results-app}. 

\begin{figure*}[t]%
    \centering
      \subfigure[$\Vert\xi\Vert_1 = 0.1$, $\rho_{1,4}=0.5$]{\includegraphics[scale=0.38]{figure/robustness_0.3_0.1_0.5.pdf}
      \label{fig:simulated_MDP_xi01_rho05}}
      \subfigure[$\Vert\xi\Vert_1 = 0.2$, $\rho_{1,4}=0.5$]{\includegraphics[scale=0.38]{figure/robustness_0.3_0.2_0.5.pdf}}
      \subfigure[$\Vert\xi\Vert_1 = 0.3$, $\rho_{1,4}=0.5$]{\includegraphics[scale=0.38]{figure/robustness_0.3_0.3_0.5.pdf}}\\
      \subfigure[$\Vert\xi\Vert_1 = 0.1$, $\rho_{1,4}=0.4$]{\includegraphics[scale=0.38]{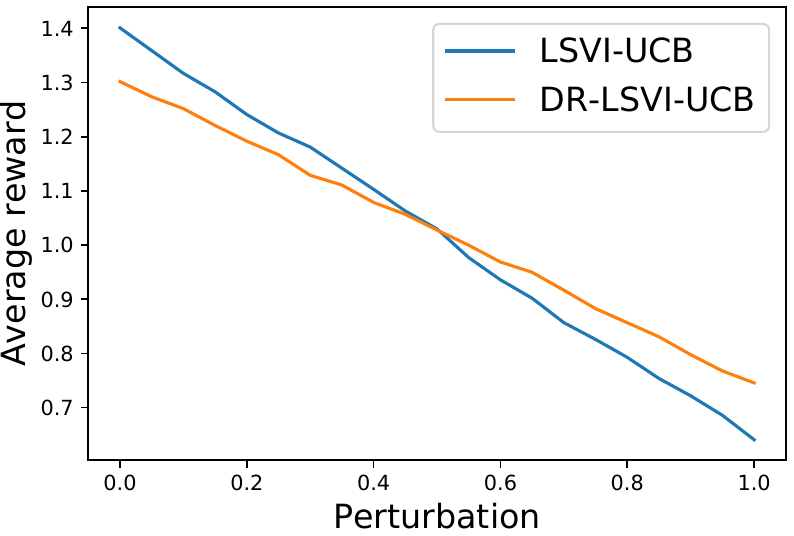}
      \label{fig:simulated_MDP_xi01_rho04}}
      \subfigure[$\Vert\xi\Vert_1 = 0.2$, $\rho_{1,4}=0.4$]{\includegraphics[scale=0.38]{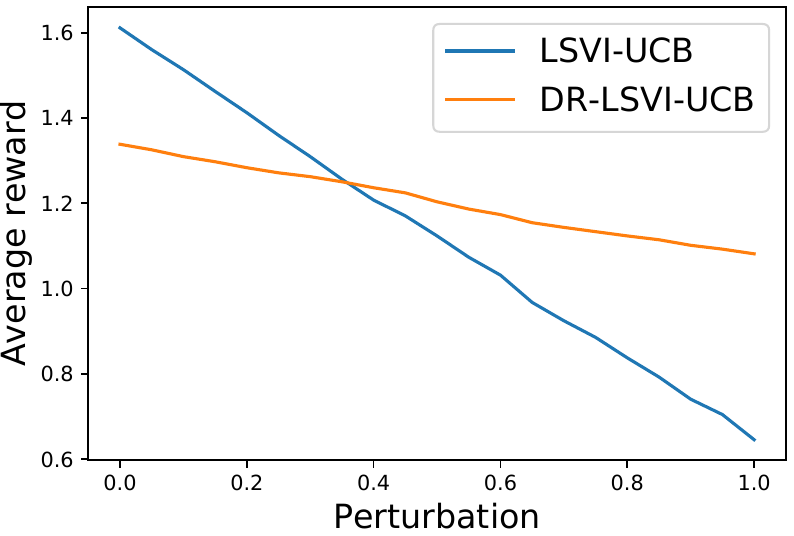}}
      \subfigure[$\Vert\xi\Vert_1 = 0.3$, $\rho_{1,4}=0.4$]{\includegraphics[scale=0.38]{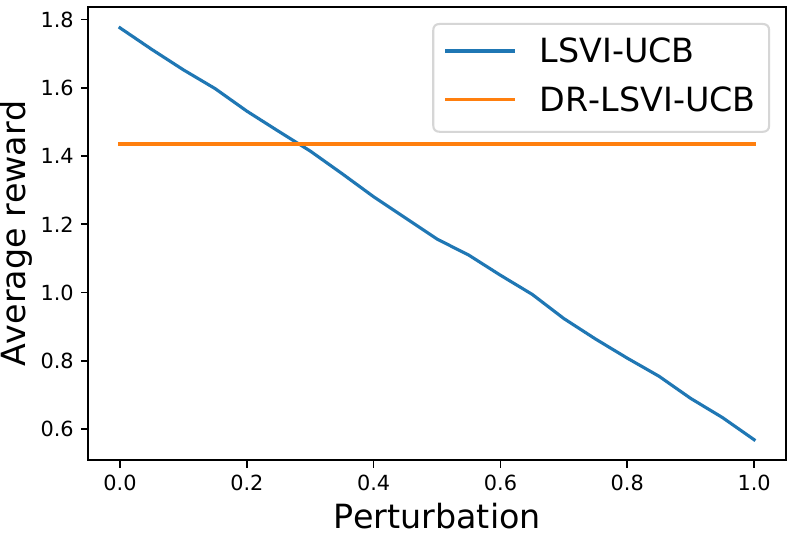}
      \label{fig:simulated_MDP_xi03_rho04}}\\
      \subfigure[$\Vert\xi\Vert_1 = 0.1$, $\rho_{1,4}=0.3$]{\includegraphics[scale=0.38]{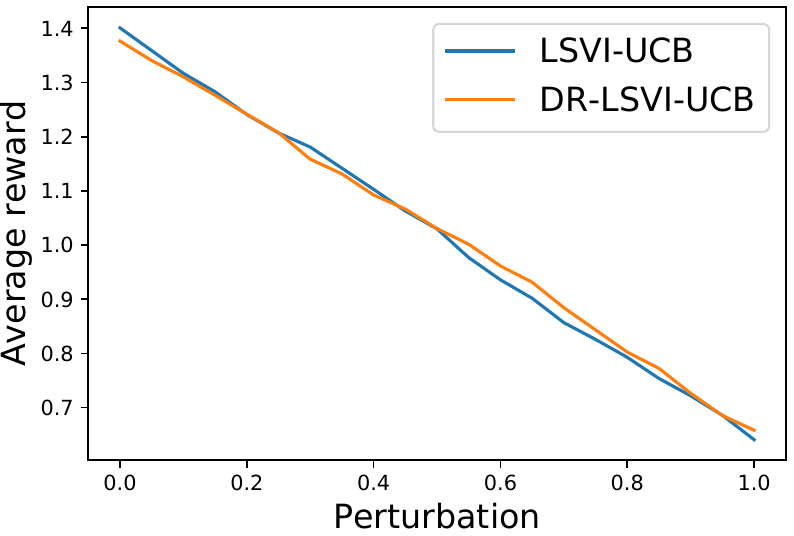}
      \label{fig:simulated_MDP_xi01_rho03}}
      \subfigure[$\Vert\xi\Vert_1 = 0.2$, $\rho_{1,4}=0.3$]{\includegraphics[scale=0.38]{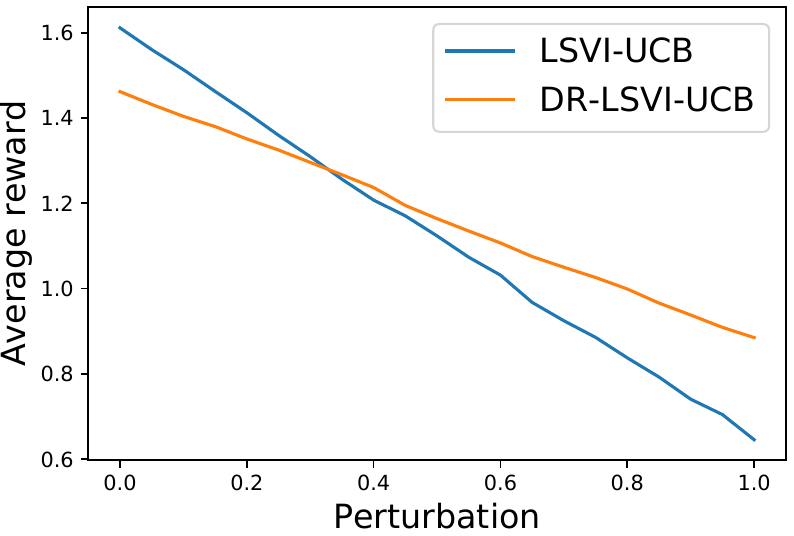}
      \label{fig:simulated_MDP_xi02_rho03}}
      \subfigure[$\Vert\xi\Vert_1 = 0.3$, $\rho_{1,4}=0.3$]{\includegraphics[scale=0.38]{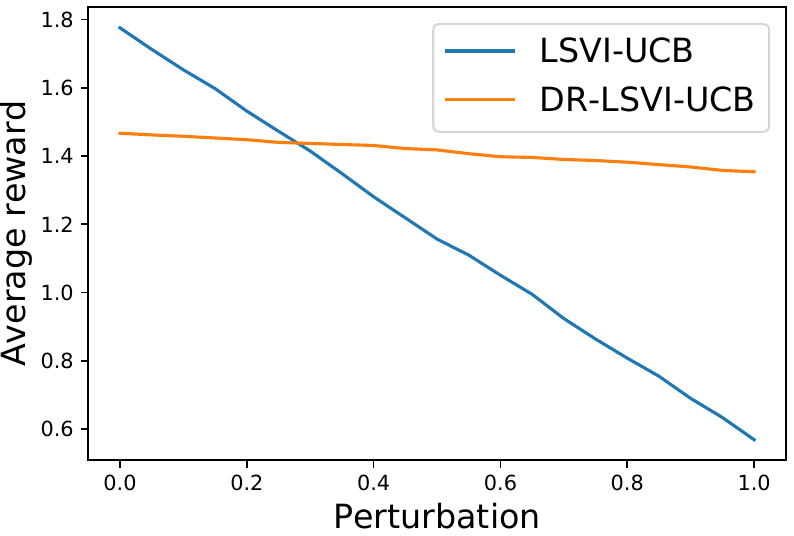}
      \label{fig:simulated_MDP_xi03_rho03}}
    \caption{Simulation results under different source domains. The $x$-axis represents the perturbation level corresponding to different target environments. $\rho_{1,4}$ is the input uncertainty level for our \algname\ algorithm.}
    \label{fig:simulation-results-app}
\end{figure*}

To interpret the results, we first delve deeper into the source linear MDP in \Cref{fig:mdp_5states}. Note that $x_5$ is an absorbing state, and $r_h(x_5, a)=1, ~\forall (h,a)\in[H]\times \cA$. For any $(s,a,h)\in \{x_1, x_2, x_3, x_4\}\times\cA\times[H]$, we have $r_h(s,a) \leq \delta+\|\xi\|_1 <1$.
Thus, the maximum reward is obtained from transitions starting from $x_5$, which can then be regarded as the goal state.
Thus, in the source domain, the optimal strategy at the first step is to take action $(1,1,1,1)$, which leads to the largest transition probability, $\delta+\|\xi\|_1$, to $x_5$. 
However, in target domains, if action $(1,1,1,1)$ is taken at the first step, it results in a probability of $(1-q)(\delta+\|\xi\|_1)$ for transitioning to state $x_5$, and also a non-negligible probability of $q(\delta+\|\xi\|_1)$ for transitioning to the fail state $x_4$.
Intuitively, when $q$ is large enough, action $(1,1,1,1)$ loses its advantage as it with high probability could cause a failure. Concretely, some calculation shows that when 
\begin{align}
\label{eq:critical point}
    q > \frac{4-2(\delta+\|\xi\|_1)(3-\delta-\|\xi\|_1)}{(4-2(\delta+\|\xi\|_1))},
\end{align}
the optimal action at the first step would be $(-1,-1,-1,-1)$, otherwise action $(1,1,1,1)$ would be the optimal action. 
Thus, the optimal policies learned in the source domain by the LSVI-UCB algorithm, which is non-robust, would fail in target domains where the perturbation level $q$ satisfies \eqref{eq:critical point}. This is consistent with our observation for all the settings in \Cref{fig:simulation-results-app}, where we see a significant performance drop of LSVI-UCB when the perturbation level increases.

In contrast, the performance of \algname\ is more robust to the dynamics shift between the source and target domains, as exemplified in \Cref{fig:simulated_MDP_xi01_rho05}. In scenarios where the MDP instance parameter $\xi$ remains the same, such as in \Cref{fig:simulated_MDP_xi01_rho03,fig:simulated_MDP_xi01_rho04,fig:simulated_MDP_xi01_rho05}, the performance of \algname\ gradually becomes more robust in the target domain as the uncertainty level, characterized by the parameter $\rho_{1,4}$, increases. This is because when $\rho_{1,4}$ is large enough, it become more likely that the uncertainty set considered by \algname\ will include the transition kernel of the target domain. This finding aligns with our theoretical analysis of the proposed \algname\ algorithm.

\section{PROOF OF MAIN RESULTS 
}
\label{sec:proof of prop}
In this section, we provide the proofs of the robust Bellman equation, the existence of the optimal robust policy, and the linear representation of the robust Q-function.
\subsection{Proof of \Cref{prop:Robust Bellman equation}} 
We first prove the robust Bellman equation for $d$-rectangular linear DRMDPs. Specifically, we will prove the following stronger statement: there exists a set of transition kernels $\tilde{P}^{\pi} = \{\tilde{P}^{\pi}_h\}_{h=1}^H$ satisfying $\tilde{P}_h^{\pi} \in \cU^{\rho}_h(P^0_h)$, such that
\begin{enumerate}
    \item Robust Bellman equation holds, 
    \begin{subequations}\label{eq:bellman_eq_stronger_result_recursion}
    \begin{align}
        V_h^{\pi, \rho}(s) &= \EE_{a\sim \pi_h(\cdot|s)}\big[Q_h^{\pi, \rho}(s,a)\big],\label{eq:bellman_eq_stronger_result_recursion_v}\\
        Q_h^{\pi, \rho}(s,a) &= r_h(s,a) + \inf_{P_h(\cdot|s,a)\in \cU_h^{\rho}(s,a;\bmu_h^0)}\EE_{s'\sim P_h(\cdot|s,a)}\big[V_{h+1}^{\pi, \rho}(s')\big].\label{eq:bellman_eq_stronger_result_recursion_q}
    \end{align}
    \end{subequations}
    \item The following expressions for robust value function and robust Q-function hold,
    \begin{subequations}\label{eq:bellman_eq_stronger_result_existence}
    \begin{align}
        V_h^{\pi, \rho}(s) &= V_h^{\pi,\{\tilde{P}_i^{\pi}\}_{i=h}^H}(s),\label{eq:bellman_eq_stronger_result_existence_v}\\
        Q_h^{\pi, \rho}(s,a) &= Q_h^{\pi,\{\tilde{P}_i^{\pi}\}_{i=h}^H }(s,a).  \label{eq:bellman_eq_stronger_result_existence_q}             
    \end{align}
    \end{subequations}    
\end{enumerate}

\begin{proof}
    We prove this proposition by induction. First, we start at the last stage $H$. The conclusion holds trivially because no transitions are involved. Suppose the conclusion holds for stage $h+1$, say there exist transition kernels $\{\tilde{P}_i^{\pi}\}_{i=h+1}^H$ such that 
    \begin{align}
    \label{eq:Prop BE - A.1}
        V_{h+1}^{\pi, \rho}(s) = V_{h+1}^{\pi, \{\tilde{P}_i^{\pi}\}_{i=h+1}^H}(s).
    \end{align}
    By the definition of $Q_h^{\pi,\rho}$, we have for any $(s,a)\in \cS \times \cA$, %
    \begin{align}
        Q_h^{\pi, \rho}(s,a) &= \inf_{P\in \cU^{\rho}(P^0)}\EE^{\{P_i\}_{i=h}^H}\Bigg[\sum_{i=h}^H r_i(s_i, a_i)\Big|s_h=s, a_h=a,\pi \Bigg]\\
        &= \inf_{P_i\in \cU_i^{\rho}(P_i^0), h\leq i\leq H}\EE^{\{P_i\}_{i=h}^H}\Bigg[\sum_{i=h}^H r_i(s_i, a_i)\Big|s_h=s, a_h=a,\pi \Bigg] \notag \\
        &= r_h(s,a)+\inf_{P_i\in \cU_i^{\rho}(P_i^0), h\leq i\leq H}\int_{\cS}P_h(ds'|s,a)\EE^{\{P_i\}_{i=h+1}^H}\Bigg[\sum_{i=h+1}^H r_i(s_i, a_i)\Big|s_{h+1}=s',\pi \Bigg] \notag \\
        &\leq r_h(s,a)+\inf_{P_h(\cdot|s,a)\in \cU_h^{\rho}(s,a;\bmu_h^0)}\int_{\cS}P_h(ds'|s,a)\EE^{\{\tilde{P}_i\}_{i=h+1}^H}\Bigg[\sum_{i=h+1}^H r_i(s_i, a_i)\Big|s_{h+1}=s',\pi \Bigg].
        \label{eq:Prop BE - A.2}
    \end{align}
    For $d$-rectangular linear DRMDP, the uncertainty sets $\{\cU_h^{\rho}(s,a;\bmu_h^0)\}_{(s,a)\in\cS\times\cA}$ are closed, and the factor uncertainty sets $\{\cU_{h,i}^{\rho}\}_{h,i=1}^{H,d}$ are decoupled from the state-action pair $(s,a)$. Thus, there exists a valid distribution $\tilde{P}_h^{\pi}$ such that for any $(s,a)\in\cS\times\cA$, %
    \begin{align}
    \label{eq:Prop BE - A.3}
        \tilde{P}_h^{\pi}(\cdot|s,a) =\arginf_{P_h(\cdot|s,a)\in \cU_h^{\rho}(s,a;\bmu^0_h)}\int_{\cS}P_h(ds'|s,a)\EE^{\{\tilde{P}_i\}_{i=h+1}^H}\Bigg[\sum_{i=h+1}^H r_i(s_i, a_i)\Big|s_{h+1}=s',\pi \Bigg].
    \end{align}
    Then by \eqref{eq:Prop BE - A.1} and the definition of $V_h^{\pi, \rho}$ and $V_h^{\pi, P}$, we have
    \begin{align}
        Q_h^{\pi, \rho}(s,a) &\leq r_h(s,a) + \inf_{P_h(\cdot|s,a) \in \cU_h^{\rho}(s,a;\bmu_h^0)}\int_{\cS}P_h(ds'|s,a)V_{h+1}^{\pi, \{\tilde{P}_i^{\pi}\}_{i=h+1}^H}(s')\label{eq:inf_p_h}\\
        &= r_h(s,a) + \inf_{P_h(\cdot|s,a) \in \cU_h^{\rho}(s,a;\bmu_h^0)}\int_{\cS}P_h(ds'|s,a)V_{h+1}^{\pi, \rho}(s') \label{eq:Prop BE - A.4}\\
        &=r_h(s,a) + \inf_{P_h(\cdot|s,a) \in \cU_h^{\rho}(s,a;\bmu_h^0)} \int_{\cS}P_h(ds'|s,a) \inf_{P_i\in \cU_i^{\rho}(P_i^0), h+1\leq i\leq H}V_{h+1}^{\pi, \{P_i\}_{i=h+1}^H}(s')\label{eq:inf_V_h+1}\\
        &=r_h(s,a)+\inf_{P_i\in \cU_i^{\rho}(P_i^0), h\leq i\leq H}\int_{\cS}P_h(ds'|s,a)V_{h+1}^{\pi, \{P_i\}_{i=h+1}^H}(s')\label{eq:Prop BE - A.5}\\
        &=r_h(s,a)+\inf_{P\in\cU^{\rho}(P^0)}\int_{\cS}P_h(ds'|s,a)V_{h+1}^{\pi, \{P_i\}_{i=h+1}^H}(s')\notag,
    \end{align}
    where \eqref{eq:inf_p_h} follows from \eqref{eq:Prop BE - A.2} and the definition of $V_{h+1}^{\pi, P}$, \eqref{eq:Prop BE - A.4} follows from \eqref{eq:Prop BE - A.1}, and \eqref{eq:inf_V_h+1} follows from the definition of $V_{h+1}^{\pi, \rho}$. Note that the RHS of \eqref{eq:Prop BE - A.5} equals to $Q_h^{\pi, \rho}(s,a)$. Therefore, all the inequalities are actually equations. On the other hand, from \eqref{eq:Prop BE - A.4} we have %
    \begin{align*}
        Q_h^{\pi, \rho}(s,a) = r_h(s,a)+\inf_{P_h(\cdot|s,a)\in \cU_h^{\rho}(s,a;\bmu_h^0)}\int_{\cS}P_h(ds'|s,a)V_{h+1}^{\pi, \rho}(s').
    \end{align*}
    This finishes the proof of Statement \eqref{eq:bellman_eq_stronger_result_recursion_q} for step $h$.

    On the other hand, by combining \eqref{eq:Prop BE - A.3} and \eqref{eq:Prop BE - A.2}, we have
    \begin{align}
    \label{eq:Prop BE - A.6}
        Q_h^{\pi, \rho}(s,a) = \EE^{\{\tilde{P}_i^{\pi}\}_{i=h}^H}\Bigg[\sum_{i=h}^Hr_i(s_i,a_i)\Big|s_h=s,a_h=a, \pi \Bigg] = Q_h^{\pi, \{\tilde{P}_i^{\pi}\}_{i=h}^H}(s,a),
    \end{align}
    which proves the existence of $\{\tilde{P}_i^{\pi}\}_{i=h}^H$ in Statement \eqref{eq:bellman_eq_stronger_result_existence_q}.

    Based on the existence of $\{\tilde{P}_i^{\pi}\}_{i=h}^H$, next we prove Statement \eqref{eq:bellman_eq_stronger_result_recursion_v} and Statement \eqref{eq:bellman_eq_stronger_result_existence_v}. By the definition of $V_h^{\pi, \rho}$, we have
    \begin{align}
        V_h^{\pi,\rho}(s) &= \inf_{P_i\in\cU_i^{\rho}(P_i^0), h\leq i\leq H}\EE^{\{P_i\}_{i=h}^H}\Bigg[\sum_{i=h}^H r_i(s_i,a_i)\Big|s_h=s, \pi\Bigg]\notag\\
        &=\inf_{P_i\in\cU_{i}^{\rho}(P_i^0), h\leq i\leq H}\sum_{a\in\cA}\pi_h(a|s)\EE^{\{P_i\}_{i=h}^H}\Bigg[\sum_{i=h}^H r_i(s_i,a_i)\Big|s_h=s,a_h=a, \pi\Bigg]\notag\\
        &\leq \sum_{a\in\cA}\pi_h(a|s)\EE^{\{\tilde{P}^{\pi}_i\}_{i=h}^H}\Bigg[\sum_{i=h}^H r_i(s_i,a_i)\Big|s_h=s,a_h=a, \pi\Bigg] \label{eq:Prop BE - A.7}.
    \end{align}
    By applying \eqref{eq:Prop BE - A.6} to \eqref{eq:Prop BE - A.7}, we further have 
    \begin{align}
        V_h^{\pi,\rho}(s) &\leq \sum_{a\in\cA}\pi_h(a|s)Q_h^{\pi, \rho}(s,a)\label{eq:Prop BE - A.8}\\
        &=\sum_{a\in\cA}\pi_h(a|s)\inf_{P_i\in\cU_i^{\rho}(P_i^0), h\leq i\leq H}\EE^{\{P_i\}_{i=h}^H}\Bigg[\sum_{i=h}^H r_i(s_i,a_i)\Big|s_h=s, a_h=a, \pi\Bigg]\label{eq:expand_Q_h}\\
        &=\inf_{P_i\in\cU_i^{\rho}(P_i^0), h\leq i\leq H}\sum_{a\in\cA}\pi_h(a|s)\EE^{\{P_i\}_{i=h}^H}\Bigg[\sum_{i=h}^H r_i(s_i,a_i)\Big|s_h=s, a_h=a, \pi\Bigg],
        \label{eq:Prop BE - A.9}
    \end{align}
    where \eqref{eq:expand_Q_h} follows form the definition of $Q_h^{\pi,\rho}$. Now note that the RHS of \eqref{eq:Prop BE - A.9} equals to $V_h^{\pi, \rho}(s)$. Therefore all the inequalities are actually equations. On the other hand, by \eqref{eq:Prop BE - A.8} we have
    \begin{align}
        \label{eq:Prop BE - A.10}
        V_h^{\pi,\rho}(s) =\sum_{a\in\cA}\pi_h(a|s)Q_h^{\pi, \rho}(s,a).
    \end{align}
    This proves \eqref{eq:bellman_eq_stronger_result_recursion_v} for stage $h$. By combining \eqref{eq:Prop BE - A.10} with \eqref{eq:Prop BE - A.6}, we further have
    \begin{align*}
        V_h^{\pi,\rho}(s) = \EE^{\{\tilde{P}^{\pi}_i\}_{i=h}^H}\Bigg[\sum_{i=h}^H r_i(s_i,a_i)\Big|s_h=s, \pi\Bigg].
    \end{align*}
    This proves Statement \eqref{eq:bellman_eq_stronger_result_existence_v} the $V_h^{\pi, \rho}$ for stage $h$. Finally, by using an induction argument, we can finish the proof of the Statement  \eqref{eq:bellman_eq_stronger_result_recursion} and \eqref{eq:bellman_eq_stronger_result_existence}. Thus, we finish the proof of \Cref{prop:Robust Bellman equation}.
\end{proof}

\subsection{Proof of \Cref{prop:Deterministic and stationary}}
We then prove the existence of the optimal robust policy for the $d$-rectangular linear DRMDP. 
\begin{proof}
    We first define a policy $\tilde{\pi}=\{\tilde{\pi}_h\}_{h=1}^H$ such that for all $h\in[H]$,
    \begin{align}
    \label{eq:Prop D&S - pi}
        \tilde{\pi}_h(s)=\argmax_{a\in\cA}\Bigg\{r_h(s,a)+\inf_{P_h(\cdot|s,a)\in \cU_h^{\rho}(s,a;\bmu_h^0)}\EE_{s'\sim P_h(\cdot|s,a)}V_{h+1}^{\star, \rho}(s)\Bigg\}.
    \end{align}
    Next we show that $\tilde{\pi}$ is optimal, i.e., for all $(h,s)\in[H]\times\cS$,
    \begin{align*}
        V_h^{\tilde{\pi},\rho}(s) = V_h^{\star, \rho}(s).
    \end{align*}
    We prove this by induction. For the last stage $H$, the conclusion holds trivially:
    \begin{align*}
        V_H^{\star, \rho}(s)=\sup_{\pi\in\Pi}V_{H}^{\pi, \rho}(s)=\sup_{\pi\in\Pi}\EE\big[r_H(s_H,a_H)|s_H=s,\pi\big] = \max_{a\in\cA}r_{H}(s,a)=V_H^{\tilde{\pi},\rho}(s).
    \end{align*}
    Now suppose that the conclusion hold for stage $h+1$, i.e., for all $s\in\cS$
    \begin{align*}
        V_{h+1}^{\tilde{\pi}, \rho}(s) = V_{h+1}^{\star, \rho}(s).
    \end{align*}
    By \Cref{prop:Robust Bellman equation}, we have
    \begin{align}
        V_h^{\tilde{\pi}, \rho}(s) &= \EE_{a\sim\tilde{\pi}_h(\cdot|s)}\Big[Q_h^{\tilde{\pi}, \rho}(s,a) \Big]\notag\\
        &=\EE_{a\sim\tilde{\pi}_h(\cdot|s)}\bigg[r_h(s,a)+\inf_{P_h(\cdot|s,a)\in\cU_h^{\rho}(s,a;\bmu^0_h)}\EE_{s'\sim P_h(\cdot|s,a)}\Big[V_{h+1}^{\tilde{\pi}, \rho}(s')\Big]\bigg]\notag\\
        &= \EE_{a\sim\tilde{\pi}_h(\cdot|s)}\bigg[r_h(s,a)+\inf_{P_h(\cdot|s,a)\in\cU_h^{\rho}(s,a;\bmu^0_h)}\EE_{s'\sim P_h(\cdot|s,a)}\big[V_{h+1}^{\star, \rho}(s)\big] \bigg]\label{eq:substitute_V_as_V_star}\\
        &= \max_{a\in\cA}\bigg[r_h(s,a)+\inf_{P_h(\cdot|s,a)\in\cU_h^{\rho}(s,a;\bmu^0_h)}\EE_{s'\sim P_h(\cdot|s,a)}\big[V_{h+1}^{\star, \rho}(s) \big]\bigg],
        \label{eq:Prop D&S - one side}
    \end{align}
    where %
    \eqref{eq:substitute_V_as_V_star} follows from the induction assumption and \eqref{eq:Prop D&S - one side} follows from the definition of $\tilde{\pi}_h$ in \eqref{eq:Prop D&S - pi}.

    On the other hand, by the definition of $V^{\star, \rho}_h(s)$, for any $s\in\cS$, we have
    \begin{align}
        V_h^{\star, \rho}(s) &= \sup_{\pi \in \Pi}V_h^{\pi, \rho}(s)\notag \\
        &=\sup_{\pi \in \Pi} \EE_{a\sim\pi_h(\cdot|s)}\big[Q_h^{\pi, \rho}(s,a)\big]\label{eq:V_to_Q}  \\
        &=\sup_{\pi \in \Pi}\EE_{a\sim\pi_h(\cdot|s)}\bigg[r_h(s,a)+\inf_{P_h(\cdot|s,a)\in \cU_h^{\rho}(s,a;\bmu^0_h)}\EE_{s'\sim P_h(\cdot|s,a)}\big[V_{h+1}^{\pi, \rho}(s')\big] \bigg]\label{eq:expand_Q_h_rho}\\
        &\leq \sup_{\pi \in \Pi}\EE_{a\sim\pi_h(\cdot|s)}\bigg[r_h(s,a)+\inf_{P_h(\cdot|s,a)\in \cU_h^{\rho}(s,a;\bmu^0_h)}\EE_{s'\sim P_h(\cdot|s,a)}\big[V_{h+1}^{\star, \rho}(s')\big] \bigg]\label{eq:substitute_V_as_V_star_rho} \\
        &=\max_{a\in\cA}\EE\bigg[r_h(s,a)+\inf_{P_h(\cdot|s,a)\in\cU_h^{\rho}(s,a;\bmu^0_h)}\EE_{s'\sim P_h(\cdot|s,a)}\big[V_{h+1}^{\star, \rho}(s) \big]\bigg]\notag,
    \end{align}
    where \eqref{eq:V_to_Q} and \eqref{eq:expand_Q_h_rho} follow from \Cref{prop:Robust Bellman equation}, \eqref{eq:substitute_V_as_V_star_rho} is due to the fact that $V_{h+1}^{\star, \rho}(s') \geq V_{h+1}^{\pi, \rho}(s'), ~\forall s'\in\cS$. Then by \eqref{eq:Prop D&S - one side}, we have $V_h^{\star, \rho}(s)\leq V_h^{\pi, \rho}(s),\forall s\in\cS$. Trivially, we also have $V_h^{\star, \rho}(s)\geq V_h^{\pi, \rho}(s) $ holds for all $s\in\cS$. Consequently, we obtain $V_h^{\star, \rho}(s)= V_h^{\pi, \rho}(s), ~\forall s\in \cS$. By using an induction argument, we finish the proof.
\end{proof}

\subsection{Proof of \Cref{prop:Linear form}}
Next, we prove that 
for any policy $\pi$, the robust Q-function $Q_h^{\pi, \rho}(\cdot,\cdot)$ is always linear with respect to the feature mapping $\bphi(\cdot, \cdot)$. 
Before presenting the proof, we first recall and define some notions. First recall the fail state that is denoted as $s_f$. The feature mapping $\tilde{\bphi}:\cS\times\cA \rightarrow \RR^{d+1}$ is defined as 
\begin{align*}
    &\tilde{\bphi}(s_f,a)=[1,0,\cdots, 0]^{\top},\quad  a\in\cA,\\
    &\tilde{\bphi}(s,a)=[0,\bphi(s,a)^{\top}]^{\top}, \quad \forall(s,a)\in \cS/\{s_f\}\times\cA.
\end{align*}
Accordingly, we define
\begin{align*}
    \tilde{\btheta}_h = \big[0, \btheta_h^{\top}\big]^{\top}, ~ \tilde{\bmu}_h^0(\cdot)=\big[\delta_{s_f}(\cdot), \bmu_h^0(\cdot)^{\top}\big]^{\top}, 
\end{align*}
where $\delta_{s_f}$ is the delta distribution with mass at $s_f$.
Then the reward function $\{\tilde{r}_h\}_{h=1}^{H}$ and nominal transition kernel $\tilde{P}^0=\{\tilde{P}^0_h\}_{h=1}^H$ have the following structures:
\begin{align}
\label{eq:linear structure with fail state}
    \tilde{r}_h(s,a) = \la \tilde{\bphi}(s,a), \tilde{\btheta}_h\ra, ~ \tilde{P}^0_h(\cdot|s,a) = \la \tilde{\bphi}(s,a), \tilde{\bmu}_h^0(\cdot)\ra, \quad \forall (h,s,a)\in [H] \times \cS \times \cA.
\end{align}

Given uncertainty level $\rho$, the uncertainty set centered around the nominal transition kernel $\{\tilde{P}^0_h\}_{h=1}^H$ is defined as 
\begin{align*}
    &\tilde{\cU}^{\rho}(\tilde{P}^0)=\bigotimes_{h\in [H]}\tilde{\cU}_h^{\rho}(\tilde{P}^0_h), ~\tilde{\cU}_h^{\rho}(\tilde{P}^0_h)=\bigotimes_{(s,a)\in\cS\times\cA}\tilde{\cU}_h^{\rho}(s,a;\tilde{\bmu}_h^0),\\
    &\tilde{\cU}_{h,i}^{\rho}(s,a;\tilde{\mu}_{h,i}^0) = \bigg\{\sum_{i=1}^{d+1}\tilde{\phi}_i(s,a)\tilde{\mu}_{h,i}(\cdot):\tilde{\mu}_{h,i}\in \tilde{\cU}_{h,i}^{\rho}(\tilde{\mu}_{h,i}^0), \forall i\in [d+1] \bigg\},\\
    &\tilde{\cU}_{h,1}^{\rho}(\tilde{\mu}_{h,1}^0)=\delta(s_f), ~
    \tilde{\cU}_{h,i}^{\rho}(\tilde{\mu}_{h,i}^0)=\big\{\tilde{\mu}_{h,i}:\tilde{\mu}_{h,i}\in \Delta(
    \cS), D_{TV}(\tilde{\mu}_{h,i}||\tilde{\mu}_{h,i}^0)\leq \rho\big\}, \quad i \in [d+1]/\{1\}.
\end{align*}
Further, we denote $[x_i]_{i\in[d]}$ as a vector with the $i$-th entry being $x_i$.
Using these notions, we are ready to prove \Cref{prop:Linear form}.
\begin{proof}
    Based on the \Cref{prop:strong duality for TV} and the linear MDP structure in \eqref{eq:linear structure with fail state}, the robust Bellman equation can be written as 
\begin{align}
    Q_h^{\pi,\rho}(s,a) &= \tilde{r}_h(s,a) + \inf_{\tilde{P}_h(\cdot|s,a)\in \tilde{\cU}_{h}^{\rho}(s,a;\tilde{\bmu}^0_h)} \EE_{s'\sim \tilde{P}_h(\cdot|s,a)}V_{h+1}^{\pi, \rho}(s')\notag\\
    & = \big\la \tilde{\bphi}(s,a),  \tilde{\btheta}_h\big\ra + \inf_{\tilde{\mu}_{h,i}\in\tilde{\cU}_{h,i}^{\rho}(\tilde{\mu}_{h,i}^0),~ i\in[d+1]}\Big\la \tilde{\bphi}(s,a), \big[\EE_{s'\sim\tilde{\mu}_{h,i}}V_{h+1}^{\pi,\rho}(s')\big]_{i\in[d+1]} \Big\ra\notag\\
    &=\bigg\la \tilde{\bphi}(s,a), \tilde{\btheta}_h+\bigg[\inf_{\tilde{\mu}_{h,i}\in\tilde{\cU}_{h,i}^{\rho}(\tilde{\mu}_{h,i}^0)}\EE_{s'\sim\tilde{\mu}_{h,i}}V_{h+1}^{\pi,\rho}(s')\bigg]_{i\in[d+1]}\bigg\ra \label{eq:move_inf_inside} \\
    &=\bigg\la \tilde{\bphi}(s,a), \tilde{\btheta}_h+\bigg[\max_{\alpha\in[0,H]}\Big\{\EE^{\tilde{\mu}_{h,i}^0}\Big[V_{h+1}^{\pi,\rho}\Big]_{\alpha}-\rho\alpha\Big\} \bigg]_{i\in[d+1]}
    \bigg\ra\notag\\
    &=\big\la \tilde{\bphi}(s,a), \tilde{\btheta}_h+\tilde{\bnu}_h^{\pi,\rho} \big\ra\notag,
\end{align}
where $\tilde{\bnu}_h^{\pi, \rho}=[\tilde{\nu}_{h,i}^{\pi, \rho}]_{i\in[d+1]}$, $\tilde{\nu}_{h,i}^{\pi, \rho}=\max_{\alpha\in[0,H]}\{\tilde{z}_{h,i}^{\pi}(\alpha)-\rho\alpha\}$,  $\tilde{z}_{h,i}^{\pi}(\alpha)=\EE^{\tilde{\mu}_{h,i}^0}[V_{h+1}^{\pi, \rho}(s') ]_{\alpha}$, and \eqref{eq:move_inf_inside} holds due to the fact that $\tilde{\bphi}(s,a)\geq 0$ and $\{\tilde{\mu}_{h,i}\}_{i\in[d+1]}$ are independent across dimensions, and thus the infimum can be moved elementwisely into the inner product. Note that $\tilde{\theta}_{h,1}=0$ and $\tilde{\nu}_{h,1}^{\pi, \rho}=0$, we have
\begin{align*}
    Q_h^{\pi,\rho}(s,a) = \big\la \phi(s,a), \btheta_h +\bnu_h^{\pi, \rho} \big\ra \ind\{s \neq s_f\},
\end{align*}
where $\bnu_h^{\pi, \rho}=[\nu_{h,i}^{\pi, \rho}]_{i\in[d]}$, $\nu_{h,i}^{\pi, \rho}=\max_{\alpha\in[0,H]}\{z_{h,i}^{\pi}(\alpha)-\rho\alpha\}$, and $z_{h,i}^{\pi}(\alpha)=\EE^{\mu_{h,i}^0}[V_{h+1}^{\pi, \rho}(s') ]_{\alpha}$.
\end{proof}

\section{PROOF OF THE MAIN RESULTS}

In this section, we provide the proofs of our main theoretical results presented in \Cref{sec:Algorithm and theoretical analysis}.

\noindent \textbf{Notation:} Throughout this section, we denote value function as $V_h^{k, \rho}(s)=\max_{a}Q_h^{k, \rho}(s,a)$, feature vector $\bphi_h^k=\bphi(s_h^k, a_h^k)$. For a vector $\bx$, we denote $(\bx)_j$ as its $j$-th entry. And we denote $[x_i]_{i\in [d]}$ as a vector with the $i$-th entry being $x_i$. 
For two $d$ dimensional vectors $\ba$ and $\bb$, we denote $\ba \leq \bb$ as the fact that $a_i-b_i \leq 0, \forall i \in [d]$.
For a matrix $A$, denote $\lambda_i(A)$ as the $i$-th eigenvalue of $A$. For two matrices $A$ and $B$, we denote $A\leq B$ as the fact that $B-A$ is a positive semidefinite matrix. 

\subsection{Proof of \Cref{th:DRLSVIUCB}}\label{sec:proof of th}
To begin with, we provide the technical lemmas that will be useful in our proof. The following concentration lemma bounds the error of the least-squares value iteration.
\begin{restatable}{lemma}{concentration}
\label{lemma:Concentration}
    Under the setting of \Cref{th:DRLSVIUCB}, let $c_{\beta}$ be the constant in our definition of $\beta$. There exists an absolute constant $C$ that is independent of $c_{\beta}$ such that for any $p\in[0,1]$, if we let $\cE$ be the event that for any $(k,h)\in [K]\times[H]$,
    \begin{align*}
        \Bigg\Vert \sum_{\tau=1}^{k-1}\bphi_h^{\tau}\bigg[\Big[V_{h+1}^{k,\rho}(s_{h+1}^{\tau})\Big]_{\alpha}-\Big[\PP_h^0\Big[V_{h+1}^{k, \rho}\Big]_{\alpha}\Big] (s_h^{\tau}, a_h^{\tau})\bigg]
        \Bigg\Vert_{(\Lambda_h^k)^{-1}}^2 \leq C\cdot d^2H^2\log[3(c_{\beta}+1)dT/p], 
    \end{align*}
    then $\PP(\cE)\geq 1-p/3$. 
\end{restatable}

The following lemma states that $Q_h^{k, \rho}$ in \Cref{alg:DR-LSVI-UCB} can always be an upper bound of $Q^{\star, \rho}_h$ with high confidence.
\begin{restatable}{lemma}{UCB}
\label{lemma:UCB}
    (UCB) Under the setting of \Cref{th:DRLSVIUCB}, on the event $\cE$ defined in \Cref{lemma:Concentration}, we have 
    \begin{align*}
        \forall (s,a,h,k) \in \cS \times \cA \times [H] \times [K],~ Q_h^{k, \rho}(s,a) \geq Q^{\star, \rho}_h(s,a).
    \end{align*}
\end{restatable}

Next, we present a recursive formula, which is useful in proving \Cref{th:DRLSVIUCB}.
\begin{restatable}{lemma}{RecursiveFormula}
\label{lemma:Recursive Formula}
(Recursive Formula) Let $\delta_{h}^{k, \rho} = V_h^{k, \rho}(s_h^k) - V_h^{\pi^k, \rho}(s_h^k)$, and
\begin{align*}
   \zeta_{h+1}^{k, \rho} = \EE_{s \sim P_h(\cdot|s_h^k, a_h^k)}\big[V_{h+1}^{k, \rho}(s) - V_{h+1}^{\pi^k, \rho}(s) \big]  - \delta_{h+1}^{k, \rho}.
\end{align*}
Then on the event defined in \Cref{lemma:Concentration}, we have the following: for any $(k,h)\in [K]\times[H]$:
\begin{align*}
    \delta_h^{k, \rho} \leq \delta_{h+1}^{k, \rho} + \zeta_{h+1}^{k, \rho} + 2\beta\sum_{i=1}^d\sqrt{\phi_{h,i}^k\mathbf{1}_i^\top (\Lambda_h^k)^{-1}\phi_{h,i}^k\mathbf{1}_i}.
\end{align*}
\end{restatable}

Finally, we are ready to prove the main theorem.
\begin{proof}[Proof of \Cref{th:DRLSVIUCB}]
    
Condition on the event $\cE$ defined in \Cref{lemma:Concentration}, by \Cref{lemma:UCB} and \Cref{lemma:Recursive Formula} we have:
\begin{align}\label{eq:AveSubopt}
    \text{AveSubopt}(K) &= \frac{1}{K}\sum_{k=1}^K\big[V_1^{\star, \rho}(s_1^k) - V_1^{\pi^k, \rho}(s_1^k) \big]\notag  \\
    &\leq \underbrace{\frac{1}{K}\sum_{k=1}^K\sum_{h=1}^{H}\zeta_h^{k, \rho}}_\text{(i)} + \underbrace{\frac{2\beta}{K}\sum_{k=1}^K\sum_{h=1}^H\sum_{i=1}^d\sqrt{\phi_{h,i}^k\mathbf{1}_i^\top (\Lambda_h^k)^{-1}\phi_{h,i}^k\mathbf{1}_i}}_\text{(ii)}.
\end{align}
For the first term (i), $\{\zeta_h^{k, \rho}\}$ is a martingale difference sequence satisfying $|\zeta_h^{k, \rho}| \leq H$ for all $(k, h)\in[K]\times[H]$. Therefore, by the Azuma-Hoeffding inequality, for any $t>0$, we have
\begin{align*}
    \PP\bigg(\sum_{k=1}^K\sum_{h=1}^H \zeta_h^{k, \rho} > t\bigg) \leq \exp\bigg(\frac{-t^2}{2KH\cdot H^2} \bigg).
\end{align*}
Hence with probability at least $1-p/3$, we have
\begin{align}
\label{eq:bound_of_first_term}
    \text{(i)}=\frac{1}{K}\sum_{k=1}^K\sum_{h=1}^H\zeta_h^{k, \rho}\leq H\sqrt{\frac{2H\log(3/p)}{K}}.
\end{align}
Maintaining the second term, thus we have
\begin{align}
\label{eq:AveSubopt-i+ii}
    \text{AveSubopt}(K)
    \leq H\sqrt{\frac{2H\log(3/p)}{K}} + \frac{2\beta}{K}\sum_{k=1}^K\sum_{h=1}^H\sum_{i=1}^d\sqrt{\phi_{h,i}^k\mathbf{1}_i^\top (\Lambda_h^k)^{-1}\phi_{h,i}^k\mathbf{1}_i}.
\end{align}
This completes the proof of \Cref{th:DRLSVIUCB}.
\end{proof}
In the rest of this section, we prove \Cref{corollary:DRLSVIUCB-tabular,corollary:DRLSVIUCB} respectively to further bound the term (ii) in \eqref{eq:AveSubopt}.

\subsection{Proof of \Cref{corollary:DRLSVIUCB-tabular}}
\begin{proof}
To prove \Cref{corollary:DRLSVIUCB-tabular}, it remains to bound the term (ii) in \eqref{eq:AveSubopt} using the structure of tabular MDP. Under tabular MDP, We set dimension $d=|\cS|\times|\cA|$ and the feature mapping $\bphi(s,a)=\be_{(s,a)}$ as the canonical basis in $\RR^d$. Define 
\begin{align*}
    N_h^k(s,a) = \sum_{\tau=1}^{k-1}\ind\{(s_h^{\tau}, a_h^{\tau})=(s,a)\},\quad \bN_h^k = [N_h^k(s,a)]_{(s,a)\in\cS\times\cA}.
\end{align*}
By the definition of feature mapping and $\Lambda_h^k$, we have 
\begin{align*}
    \Lambda_h^k=\sum_{\tau=1}^{k-1}\bphi_{h}^{\tau}(\bphi_h^{\tau})^{\top}+\lambda I =\diag(\bN_h^{k}+\lambda\mathbf{1}),
\end{align*}
where $\mathbf{1}$ is the vector with all entries being 1. By our choice of $\lambda$, we  have
\begin{align}
    \text{(ii)} &= \frac{2\beta}{K}\sum_{k=1}^K\sum_{h=1}^H\frac{1}{\sqrt{N_h^{k}(s_h^k, a_h^k)+1}} \notag\\ 
    &\leq \frac{2\beta}{K}\sum_{h=1}^H\sum_{k=1}^K\frac{1}{\sqrt{N_h^{k}(s_h^k, a_h^k)}} \notag \\
    &= \frac{2\beta}{K}\sum_{h=1}^H\sum_{(s,a)\in\cS\times\cA}\sum_{i=1}^{N_h^K(s,a)}\frac{1}{\sqrt{i}}\notag\\
    &\leq\frac{4\beta}{K}\sum_{h=1}^H\sum_{(s,a)\in\cS\times\cA}\sqrt{N_h^K(s,a)}\label{eq:bound_summation_of_one_over_square_root_i}\\
    &\leq \frac{4\beta}{K}\sum_{h=1}^H\sqrt{SA\sum_{(s,a)\in\cS\times\cA}N_h^K(s,a)}\label{eq:put_sum_of_SA_in_square_root}\\
    &=\frac{4\beta}{K}H\sqrt{SAK},\label{eq:ii-tabular}
\end{align}
where \eqref{eq:bound_summation_of_one_over_square_root_i} follows from the fact that $\sum_{i=1}^N\frac{1}{\sqrt{i}}\leq \sqrt{N}$, \eqref{eq:put_sum_of_SA_in_square_root} follows from Cauchy-Schwarz inequality. Substitute \eqref{eq:ii-tabular} into \eqref{eq:AveSubopt-i+ii} and with our choice of $\beta=c_{\beta}\cdot dH\sqrt{\log3dHK/p}$ and the fact $d=SA$ we have 
\begin{align*}
    \text{AveSubopt}(K) &\leq H\sqrt{\frac{2H\log(3/p)}{K}} + \frac{4c_{\beta}\cdot SAH\sqrt{\log3SAHK/p}}{K}H\sqrt{SAK}\\
    &\leq \frac{c(SA)^{3/2}H^2\sqrt{\log3SAHK/p}}{\sqrt{K}},
\end{align*}
which completes the proof.
\end{proof}

\subsection{Proof of \Cref{corollary:DRLSVIUCB}}
The proof of this corollary requires the following concentration inequality.
\begin{lemma}\cite[Matrix Azuma inequality]{tropp2012user}
Consider a finite adapted sequence $\{X_k\}$ of self-adjoint matrices in dimension $d$, and a fixed sequence $\{A_k\}$ of self-adjoint matrices that satisfy
\begin{align*}
    \EE_{k-1}[X_k]=0 ~ \text{and} ~ X_k^2 \leq A_k^2 ~\text{almost surely}.
\end{align*}
Compute the variance parameter %
\begin{align*}
    \sigma^2 := \bigg\Vert\sum_kA^2_k \bigg \Vert.
\end{align*}
Then, for all $t\geq 0$, %
\begin{align*}
    \PP\bigg\{\lambda_{\max}\bigg(\sum_k X_k \bigg) \geq t \bigg\} \leq d\cdot e^{-t^2/8\sigma^2}.
\end{align*}
\label{lemma:Matrix Azuma}
\end{lemma}

\begin{proof}[Proof of \Cref{corollary:DRLSVIUCB}]
Based on the proof of \Cref{th:DRLSVIUCB}, it remains to bound the term (ii) in \eqref{eq:AveSubopt} using the condition in \eqref{eq:Feature exploration}. By Cauchy–Schwarz inequality we have
\begin{align}
    \text{(ii)}\cdot K &= 2\beta\sum_{k=1}^K\sum_{h=1}^H\sum_{i=1}^d\sqrt{\phi_{h,i}^k\mathbf{1}_i^\top (\Lambda_h^k)^{-1}\phi_{h,i}^k\mathbf{1}_i}\notag \\
    &= 2\beta\sum_{k=1}^K\sum_{h=1}^H \sum_{i=1}^d\phi_{h,i}^k\sqrt{\mathbf{1}_i^\top (\Lambda_h^k)^{-1}\mathbf{1}_i}\notag \\
    & \leq 2\beta\sum_{k=1}^K\sum_{h=1}^H \sum_{i=1}^d\phi_{h,i}^k\sqrt{\lambda_{\max}\big((\Lambda_h^k)^{-1} \big)}\label{eq:bound_diagonal_by_max_eignevalue}\\
    &=2\beta\sum_{k=1}^K\sum_{h=1}^H \sqrt{\lambda_{\max}\big((\Lambda_h^k)^{-1} \big)}\notag \\
    &= 2\beta\sum_{k=1}^K\sum_{h=1}^H\sqrt{\frac{1}{\lambda_{\text{min}}(\Lambda_h^k)}}\notag\\
    &\leq 2\beta\sqrt{K}\sum_{h=1}^H\sqrt{\sum_{k=1}^K\frac{1}{\lambda_{\text{min}}(\Lambda_h^k)}},\notag
\end{align}
where \eqref{eq:bound_diagonal_by_max_eignevalue} follows by the fact for any matrix $\bA$, $\lambda_{\min}\leq \bA_{ii} \leq \lambda_{\max}$, where $\bA_{ii}$ is the $i$-th diagonal element of $\bA$.

Next we bound $\lambda_{\text{min}}(\Lambda_h^k)$. First, fix $(k,h)\in[K]\times[H]$. Recall that $\Lambda_h^k=\sum_{\tau=1}^{k-1}\bphi_h^{\tau}(\bphi_h^{\tau})^\top+\lambda I$, we have
\begin{align*}
    \Lambda_h^k-\EE\big[\Lambda_h^k\big] = \sum_{\tau=1}^{k-1} \big[\bphi_h^{\tau}(\bphi_h^{\tau})^\top - \EE_{\pi^{\tau}}\big[\bphi_h^{\tau}(\bphi_h^{\tau})^\top\big]\big] = \sum_{\tau=1}^{k-1}X_h^{\tau},
\end{align*}
where $X_h^{\tau} = \bphi_h^{\tau}(\bphi_h^{\tau})^\top - \EE_{\pi^{\tau}}[\bphi_h^{\tau}(\bphi_h^{\tau})^\top]$. Then $\{X_h^{\tau}\}$ is a matrix martingale difference sequence. Note that $ \Vert \bphi_h^{\tau}(\bphi_h^{\tau})^\top\Vert_{\text{op}} \leq 1$, then we have
\begin{align*}
   \Vert X_h^{\tau}\Vert_{\text{op}} \leq \Vert\bphi_h^{\tau}(\bphi_h^{\tau})^\top \Vert_{\text{op}} + \Vert\EE_{\pi^{\tau}}\big[\bphi_h^{\tau}(\bphi_h^{\tau})^\top\big]\Vert_{\text{op}} \leq 1+\EE_{\pi^{\tau}}\big[\Vert\bphi_h^{\tau}(\bphi_h^{\tau})^\top \Vert_{\text{op}} \big] \leq 2,
\end{align*}
so $\Vert (X_h^{\tau})^2\Vert_{\text{op}} \leq \Vert X_h^{\tau}\Vert_{\text{op}}^2\leq 4$. Then we have $(X_h^{\tau})^2 \leq 4I$ and $\sigma^2 := \Vert \sum_{\tau=1}^{k-1}4I\Vert_{\text{op}}=4(k-1)$. By \Cref{lemma:Matrix Azuma}, for any $t_k\geq 0$ we have
\begin{align*}
    \PP\bigg\{\lambda_{\max}\bigg(-\sum_{\tau=1}^{k-1} X_h^{\tau} \bigg) \geq t_k \bigg\} \leq d\cdot e^{-t_k^2/32(k-1)}.
\end{align*}
Let $t_k=\sqrt{32k\log(3d/\delta)}$, then with probability at least $1-\delta/3$, we have
\begin{align*}
    \sum_{\tau=1}^{k-1} X_h^{\tau} \geq -t_kI.
\end{align*}
Let $\delta=p/KH$ and define 
\begin{align*}
    \cE^{\dagger} = \Bigg\{\sum_{\tau=1}^{k-1} X_h^{\tau} \geq -t_kI: \forall (k, h)\in [K]\times[H] \Bigg\},
\end{align*}
then by union bound we have $\PP(\cE^{\dagger})\geq 1-p/3$.

By \eqref{eq:Feature exploration}, we have 
\begin{align*}
\EE\big[\Lambda_h^k\big]=\sum_{\tau=1}^k\EE_{\pi^{\tau}}\big[\bphi_h^{\tau}(\bphi_h^{\tau})^{\top} + \lambda I \big]\geq \alpha(k-1)I+\lambda I.
\end{align*}
Condition on $\cE^{\dagger}$,  we have
\begin{align*}
    \Lambda_h^k=\Lambda_h^k-\EE\Lambda_h^k+\EE\Lambda_h^k \geq -t_kI+\EE\Lambda_h^k.
\end{align*}
Thus, we have
\begin{align*}
    \lambda_{\min}(\Lambda_h^k) \geq \max\big\{\alpha(k-1)+\lambda-\sqrt{32k\log(3dKH/p)}, \lambda \big\}.
\end{align*}
By our choice of $\lambda$, then we have
\begin{align}
    \sum_{k=1}^K\frac{1}{\lambda_{\min}(\Lambda_h^k)} &\leq \sum_{k=1}^K\frac{1}{\max\{\alpha(k-1)+1-\sqrt{32k\log(3dHK/p)}, 1\}}\notag \\
    &\leq \frac{128}{\alpha^2}\log\frac{3dHK}{p} + \sum_{k=1}^K\frac{2}{\alpha\cdot k}\notag \\
    &\leq \frac{128}{\alpha^2}\log\frac{3dHK}{p} + \frac{2}{\alpha}\log K\label{eq:bound_sum_of_one_over_k},
\end{align}
where \eqref{eq:bound_sum_of_one_over_k} follows from the fact that $\sum_{k=1}^K 1/k \leq \log K$. Therefore the term (ii) can be bounded as 
\begin{align}
\label{eq:second term}
    \text{(ii)} \leq 2\beta\sum_{h=1}^H\sqrt{\frac{1}{K}\sum_{k=1}^K\frac{1}{\lambda_{\min}(\Lambda_h^k)}} \leq 2H\frac{\beta}{\sqrt{K}}\sqrt{\frac{128}{\alpha^2}\log\frac{3dHK}{p} + \frac{2}{\alpha}\log K}.
\end{align}
Finally combining \eqref{eq:AveSubopt}, \eqref{eq:bound_of_first_term} and \eqref{eq:second term} and with our choice of $\beta=c_{\beta}\cdot dH\sqrt{\log{3dKH/p}}$, we conclude that with probability $1-p$:
\begin{align*}
    \text{AveSubopt}(K) &\leq \frac{2H\sqrt{H}\log(3/p)}{\sqrt{K}} +  \frac{2\beta H}{\sqrt{K}}\sqrt{\frac{128}{\alpha^2}\log\frac{3dHK}{p} + \frac{2}{\alpha}\log K }\leq \frac{cdH^2\log(3dHK/p)}{\alpha\sqrt{K}},
\end{align*}
for some absolute constant $c$. This concludes the proof.
\end{proof}

\section{PROOF OF TECHNICAL LEMMAS}
\label{sec:proof of technical lemmas used in prove th}
\subsection{Proof of \Cref{lemma:Concentration}}
In this section, we prove \Cref{lemma:Concentration}. 
Before the proof, we first present several auxiliary lemmas.

The following lemma states that the linear weights  in \Cref{alg:DR-LSVI-UCB} are bounded.
\begin{restatable}{lemma}{BoundedWeights}
\label{lemma:Bound of weights}
    For any $(k,h) \in [K]\times[H]$, denote the weight $\bw_h^{\rho, k} = \btheta_h+\bnu_h^{\rho,k}$ in \Cref{alg:DR-LSVI-UCB}, then $\bw_h^{\rho, k}$ satisfies
    \begin{align*}
        \Vert \bw_h^{\rho, k} \Vert_2 \leq 2H\sqrt{dk/\lambda}.
    \end{align*}
\end{restatable}

The following lemma presents a uniform self-normalized concentration over all value functions $V$ within a function class $\cV$ and all parameters $\alpha$ with the interval $[0,H]$.
\begin{restatable}{lemma}{UniformConcentrationBound}
\label{lemma:Uniform concentration bound}
    Let $\{x_{\tau}\}_{\tau=1}^{\infty}$ be a stochastic process on the state space $\cS$ with corresponding filtration $\{\cF_{\tau}\}_{\tau=0}^{\infty}$. Let $\{\bphi_{\tau}\}_{\tau=1}^{\infty}$ be an $\RR^d$-valued stochastic process with $\bphi_{\tau} \in \cF_{\tau-1}$, and $\Vert \bphi_{\tau}\Vert\leq 1$. Let $\Lambda_k=\lambda I + \sum_{\tau=1}^{k-1}\bphi_{\tau}\bphi_{\tau}^\top$, then for any $\delta>0$, with probability at least $1-\delta$, for all $k\geq 0$, any $\alpha\in[0,H]$ and any $V\in \cV$ such that $\sup_{x}|V(x)|\leq H$, we have
    \begin{align*}
        \bigg\Vert \sum_{\tau=1}^k\bphi_{\tau}\big\{\big[V(x_{\tau})\big]_{\alpha}-\EE\big[\big[V(x_{\tau})\big]_{\alpha} |\cF_{\tau-1}\big]\big\}\bigg\Vert^2_{\Lambda_k^{-1}} &\leq 8H^2\bigg[\frac{d}{2}\log\frac{k+\lambda}{\lambda}+\log\frac{\cN_{\epsilon_1}}{\delta} + \log\frac{\cN_{\epsilon_2}}{\delta}\bigg] \notag\\
        &\qquad+ \frac{16k^2\epsilon_1^2}{\lambda}+\frac{8k^2\epsilon_2^2}{\lambda},
    \end{align*}
    where $\cN_{\epsilon_1}$ is the $\epsilon_1$ covering number of the interval $[0, H]$ with respect to the distance $\dist(\alpha_1, \alpha_2)=|\alpha_1- \alpha_2|$, and $\cN_{\epsilon_2}$ is the $\epsilon_2$ covering number of $\cV$ with respect to the distance $\dist(V_1, V_2)=\sup_{x}|V_1(x) - V_2(x)|$.
\end{restatable}

\begin{lemma}
\label{lemma:Covering number of the function class V}
    (Covering number of the function class $\cV$) Let $\cV$ denote a class of functions mapping from $\cS$ to $\RR$ with the following parametric form %
    \begin{align*}
        V(\cdot)=\min \bigg\{\max_{a}\bigg\{w^{\top}\bphi(\cdot,a)+\beta\sum_{i=1}^d\sqrt{\phi_i(\cdot,a)\mathbf{1}_i^{\top}\Lambda^{-1}\phi_i(\cdot,a)\mathbf{1}_i}\bigg\},H \bigg\},
    \end{align*}
    where the parameters $(w, \beta, \Lambda, \alpha)$ satisfy $\Vert w\Vert \leq L$, $\beta \in [0, B]$, $\lambda_{\min}(\Lambda)\geq \lambda$ and $\alpha\in[0,H]$. Assume $\Vert \bphi(s,a)\Vert\leq 1$ for all (s,a) pairs, and let $\cN_{\epsilon}$ be the $\epsilon$-covering number of $\cV$ with respect to the distance $\dist(V_1, V_2)=\sup_x|V_1(x)-V_2(x)|$. Then 
    \begin{align*}
        \log\cN_{\epsilon} \leq d\log(1+4L/\epsilon) + d^2\log\big[1+8d^{1/2}B^2/(\lambda\epsilon^2) \big].
    \end{align*}
\end{lemma}

\begin{lemma}\cite[Covering number of an interval]{vershynin2018high}
\label{lemma:Covering number of an interval}
     Denote the $\epsilon$-covering number of the closed interval $[a,b]$ for some real number $b>a$ with respect to the distance metric $d(\alpha_1, \alpha_2)=|\alpha_1-\alpha_2|$ as $\cN_{\epsilon}([a,b])$. Then we have $\cN_{\epsilon}([a,b])\leq 3(b-a)/\epsilon$.
\end{lemma}

\begin{proof}[Proof of \Cref{lemma:Concentration}]
    For all $(k,h)\in[K]\times[H]$, by \Cref{lemma:Bound of weights} we have $\Vert w_h^{\rho, k}\Vert\leq 2H\sqrt{dk/\lambda}$. By the construction of $\Lambda_h^k$, the minimum eigenvalue of $\Lambda_h^k$ is lower bounded by $\lambda$. By combining \Cref{lemma:Uniform concentration bound,lemma:Covering number of an interval,lemma:Covering number of the function class V}, for any fix $\epsilon>0$ , set $\epsilon_1=\epsilon_2=\epsilon$, we have
    \begin{align}
        &\Bigg\Vert\sum_{\tau=1}^{k-1}\bphi_{h}^{\tau}\bigg[\Big[V_{h+1}^{k, \rho}(s_{h+1}^{\tau})\Big]_{\alpha} -\Big[\PP_h^0\Big[V_{h+1}^{k, \rho}\Big]_{\alpha}\Big] (s_h^{\tau}, a_h^{\tau})\bigg]\Bigg\Vert^2_{(\Lambda_h^k)^{-1}}\notag\\
        &\leq 4H^2\bigg[\frac{d}{2}\log\frac{k+\lambda}{\lambda}+d\log\bigg(1+\frac{8H\sqrt{dk}}{\epsilon\sqrt{\lambda}}\bigg)+d^2\log\bigg(1+\frac{8d^{1/2}\beta^2}{\epsilon^2\lambda}\bigg)+\log\frac{3H}{\epsilon}+\log\frac{3}{p}\bigg] + \frac{24k^2\epsilon^2}{\lambda}.
        \label{eq:concentration inequality}
    \end{align}
    
    In \Cref{alg:DR-LSVI-UCB}, we choose parameters $\lambda=1$ and $\beta=c_{\beta}dH\iota$, where $c_{\beta}$ is an absolute constant. Finally, picking $\epsilon=dH/k$, by \eqref{eq:concentration inequality}, there exists an absolute $C>0$ that is independent of $c_{\beta}$ such that 
    \begin{align*}
        &\Bigg\Vert\sum_{\tau=1}^{k-1}\bphi_{h}^{\tau}\bigg[\Big[V_{h+1}^{k, \rho}(s_{h+1}^{\tau})\Big]_{\alpha} -\Big[\PP_h^0\Big[V_{h+1}^{k, \rho}\Big]_{\alpha}\Big] (s_h^{\tau}, a_h^{\tau})\bigg]\Bigg\Vert^2_{(\Lambda_h^k)^{-1}} \leq C\cdot d^2H^2\log\frac{3(c_{\beta}+1)dKH}{p},
    \end{align*}
    which completes the proof.
\end{proof}

\subsection{Proof of \Cref{lemma:UCB}}
Before the proof of \Cref{lemma:UCB}, we present a lemma bounding the difference between the value function maintained in \Cref{alg:DR-LSVI-UCB} (without
bonus) and the true value function of any policy $\pi$.

\begin{restatable}{lemma}{BoundedDifference}
\label{lemma:Bounded difference}
    For any fixed policy $\pi$, on the event $\cE$ defined in \Cref{lemma:Concentration}, we have for all $(s,a,h,k) \in \cS/\{s_f\} \times \cA \times [H] \times [K]$ that:
    \begin{align*}
        \la \bphi(s,a), \btheta_h+\bnu_{h}^{\rho, k}\ra - Q_h^{\pi, \rho}(s,a) &= \inf_{P_h(\cdot|s,a) \in \cU_{h}^{\rho}(s,a;\bmu_h^0)}\big[\PP_h V_{h+1}^{k, \rho}\big](s,a)  \\
        &\qquad- \inf_{P_h(\cdot|s,a) \in \cU_{h}^{\rho}(s,a;\bmu_h^0)} \big[\PP_h V_{h+1}^{\pi, \rho}\big](s,a) + \Delta_h^k(s,a),
    \end{align*}
for some $\Delta_h^k(s,a)$ that satisfies $|\Delta_h^k(s,a)|\leq \beta \sum_{i=1}^d\sqrt{\phi_i(s,a)\mathbf{1}_i^{\top}(\Lambda_h^k)^{-1}\phi_i(s,a)\mathbf{1}_i}$.
\end{restatable}

\begin{proof}[Proof of \Cref{lemma:UCB}]
    We prove this lemma by induction.
  Starting at step $H-1$. Since $V_{H}^{k, \rho}(s) = V_{H}^{\star, \rho}(s)=\max_{a}r_H(s,a)$, by \Cref{lemma:Bounded difference} we have
    \begin{align*}
        \big| \big\la \bphi(s,a), \btheta_{H-1} + \bnu_{H-1}^{\rho, k}\big\ra - Q^{\star, \rho}_H(s,a)\big| \leq \Gamma_{H-1}^k(s,a),
    \end{align*}
    where $\Gamma_{H-1}^k(s,a)$ is the bonus at step $H-1$ used in \Cref{alg:DR-LSVI-UCB}.
    Therefore, we know
    \begin{align*}
        Q_{H-1}^{k, \rho}=\min\big\{\big\la \bphi(s,a), \btheta_{H-1} + \bnu_{H-1}^{\rho, k}\big\ra + \Gamma_{H-1}^k(s,a), H \big \} \geq Q_{H-1}^{\star, \rho}(s,a).
    \end{align*}
    Suppose the statement holds at stage $h+1$, $Q_{h+1}^{k, \rho}(s,a) \geq Q_{h+1}^{\star, \rho}(s,a)$ for any $(s,a)\in\cS\times\cA$, then we have
    \begin{align*}
        V_{h+1}^{k, \rho}(s) = Q_{h+1}^{k, \rho}(s, \pi_{h+1}^{k}(s)) \geq Q_{h+1}^{k, \rho}(s, \pi_{h+1}^{\star}(s)) \geq Q_{h+1}^{\star, \rho}(s, \pi_{h+1}^{\star}(s)) = V_{h+1}^{\star, \rho}(s), \quad \forall s\in \cS,
    \end{align*}
    where the first inequality holds by the fact that $\pi_{h+1}^k$ is the greedy policy with respect to $Q_{h+1}^{k, \rho}$, and the second inequality holds by the induction assumption that $Q_{h+1}^{k, \rho}(s,a) \geq Q_{h+1}^{\star, \rho}(s,a), ~\forall (s,a)\in\cS\times\cA$. Thus, we have
    \begin{align}
        \inf_{P_h(\cdot|s,a) \in \cU_{h}^{\rho}(s,a;\bmu_h^0)}\big[\PP_hV_{h+1}^{k, \rho}\big](s,a) - \inf_{P_h(\cdot|s,a) \in \cU_{h}^{\rho}(s,a;\bmu_h^0)}\big[\PP_hV_{h+1}^{\star, \rho}\big](s,a) \geq 0. \label{eq:induction assumption}
    \end{align}
    Again by \Cref{lemma:Bounded difference} we have 
    \begin{align*}
        &\Big|\big\la \bphi(s,a), \btheta_h+\bnu_{h}^{\rho, k}\big\ra - Q_h^{\star, \rho}(s,a) -\Big( \inf_{P_h(\cdot|s,a) \in \cU_{h}^{\rho}(s,a;\bmu_h^0)}\big[\PP_hV_{h+1}^{k, \rho}\big](s,a) - \inf_{P_h(\cdot|s,a) \in \cU_{h}^{\rho}(s,a;\bmu_h^0)}\big[\PP_hV_{h+1}^{\star, \rho}\big](s,a) \Big) \Big| \\
        &\leq \beta \sum_{i=1}^d\sqrt{\phi_i(s,a)\mathbf{1}_i^{\top}(\Lambda_h^k)^{-1}\phi_i(s,a)\mathbf{1}_i}.
    \end{align*}
    By \eqref{eq:induction assumption} we have 
    \begin{align*}
        Q_h^{k, \rho}(s,a) = \min\big\{\big\la \bphi(s,a), \btheta_{h} + \bnu_{h}^{\rho, k}\big\ra + \Gamma_{h}^k(s,a), H - h +1 \big\} \geq Q_h^{\star, \rho}(s,a),
    \end{align*}
    which concludes the proof.
\end{proof}

\subsection{Proof of \Cref{lemma:Recursive Formula}}
\begin{proof}
    By \Cref{alg:DR-LSVI-UCB} and the definition of $\pi^k$, we have
    \begin{align*}
        \delta_h^{k, \rho} = V_h^{k, \rho}(s_h^k) - V_h^{\pi^k, \rho}(s_h^k) = Q_h^{k, \rho}(s_h^k, a_h^k) - Q_h^{\pi^k, \rho}(s_h^k, a_h^k).
    \end{align*}
    By \Cref{lemma:Bounded difference} we have
    \begin{align}
    \label{eq:recursive formula}
        \delta_h^{k, \rho} &\leq \inf_{P_h(\cdot|s_h^k,a_h^k) \in \cU_{h}^{\rho}(s_h^k,a_h^k;\bmu_h^0)}\big[\PP_hV_{h+1}^{k, \rho}\big](s_h^k,a_h^k)\notag  - \inf_{P_h(\cdot|s_h^k,a_h^k) \in \cU_{h}^{\rho}(s_h^k,a_h^k;\bmu_h^0)}\big[\PP_hV_{h+1}^{\pi^k, \rho}\big](s_h^k,a_h^k)\\
        &\qquad + 2\beta \sum_{i=1}^d\sqrt{\phi_{h,i}^k\mathbf{1}_i^{\top}(\Lambda_h^k)^{-1}\phi_{h,i}^k\mathbf{1}_i}.
    \end{align}
    For the difference on the RHS, we have
    \begin{align*}
        &\inf_{P_h(\cdot|s_h^k,a_h^k) \in \cU_{h}^{\rho}(s_h^k,a_h^k;\bmu_h^0)}\big[\PP_hV_{h+1}^{k, \rho}\big](s_h^k,a_h^k)\notag  - \inf_{P_h(\cdot|s_h^k,a_h^k) \in \cU_{h}^{\rho}(s_h^k,a_h^k;\bmu_h^0)}\big[\PP_hV_{h+1}^{\pi^k, \rho}\big](s_h^k,a_h^k)\\
        &= \bigg\la \bphi(s_h^k, a_h^k), \bigg[\max_{\alpha_i\in[0,H]}\Big\{\EE^{\mu_{h,i}^0}\big[V_{h+1}^{k, \rho}(s) \big]_{\alpha_i} - \rho\alpha_i \Big\} \bigg]_{i\in [d]} \bigg\ra \notag\\
        &\qquad- \bigg\la \bphi(s_h^k, a_h^k), \bigg[\max_{\alpha_i\in[0,H]}\Big\{\EE^{\mu_{h,i}^0}\big[V_{h+1}^{\pi^k, \rho}(s) \big]_{\alpha_i} - \rho\alpha_i \Big\} \bigg]_{i \in [d]}\bigg\ra\\
        &\leq \bigg\la \bphi(s_h^k, a_h^k), \bigg[\max_{\alpha_i\in[0,H]}\Big\{\EE^{\mu_{h,i}^0}\big[V_{h+1}^{k, \rho}(s) \big]_{\alpha_i} - \EE^{\mu_{h,i}^0}\big[V_{h+1}^{\pi^k, \rho}(s) \big]_{\alpha_i} \Big\} \bigg]_{i\in [d]} \bigg\ra.
    \end{align*}
    By \Cref{lemma:UCB}, we have for all $s\in \cS$,
    \begin{align*}
    V_{h+1}^{k, \rho}(s) = Q_{h+1}^{k, \rho}(s,
    \pi_{h+1}^k(s)) \geq  Q_{h+1}^{k, \rho}(s,
    \pi_{h+1}^{\star}(s))\geq Q_{h+1}^{\star, \rho}(s,
    \pi_{h+1}^{\star}(s)).
    \end{align*}
    Since $\pi^{\star}$ is the greedy policy with respect to $Q_{h+1}^{\star, \rho}$, we have
    \begin{align*}
    V_{h+1}^{k, \rho}(s)  \geq Q_{h+1}^{\star, \rho}(s,
    \pi_{h+1}^k(s)) \geq Q_{h+1}^{\pi^k, \rho}(s,
    \pi_{h+1}^k(s))= V_{h+1}^{\pi^k, \rho}(s).
    \end{align*}
    Then we have,
    \begin{align}
        &\inf_{P_h(\cdot|s_h^k,a_h^k) \in \cU_{h}^{\rho}(s_h^k,a_h^k;\bmu_h^0)}\big[\PP_hV_{h+1}^{k, \rho}\big](s_h^k,a_h^k)\notag  - \inf_{P_h(\cdot|s_h^k,a_h^k) \in \cU_{h}^{\rho}(s_h^k,a_h^k;\bmu_h^0)}\big[\PP_hV_{h+1}^{\pi^k, \rho}\big](s_h^k,a_h^k)\notag\\
        &\leq \big\la \bphi(s_h^k, a_h^k), \EE^{\bmu_{h}^0}\big[V_{h+1}^{k, \rho}(s) - V_{h+1}^{\pi^k, \rho}(s)\big]\big\ra  \notag\\
        &=\big[\PP_h\big[V_{h+1}^{k, \rho} - V_{h+1}^{\pi^k, \rho}\big]\big](s_h^k,a_h^k) \notag\\
        &=\big[\PP_h\big[V_{h+1}^{k, \rho} - V_{h+1}^{\pi^k, \rho}\big]\big](s_h^k,a_h^k) - \big[V_{h+1}^{k, \rho}(s_{h+1}^k) - V_{h+1}^{\pi^k, \rho}(s_{h+1}^k)\big] + \big[V_{h+1}^{k, \rho}(s_{h+1}^k) - V_{h+1}^{\pi^k, \rho}(s_{h+1}^k)\big]  \notag\\
        &=\zeta_{h+1}^{k, \rho} + \delta_{h+1}^{k, \rho}. \label{eq:first term}
    \end{align}
    Then we complete the proof by substituting \eqref{eq:first term} into \eqref{eq:recursive formula}.
\end{proof}

\section{PROOF OF SUPPORTING LEMMAS}
In this section, we provide the proofs of the supporting lemmas we used in \Cref{sec:proof of technical lemmas used in prove th}.

\subsection{Proof of \Cref{lemma:Bound of weights}}

The proof of \Cref{lemma:Bound of weights} will use the following fact. 
\begin{lemma}\cite[Lemma D.1]{jin2020provably}
\label{lemma:self-normalize}
    Let $\Lambda_t=\lambda \bI + \sum_{i=1}^t\bphi_i\bphi_i^{\top}$, where $\bphi_i\in\RR^d$ and $\lambda > 0$. Then:
    \begin{align*}
        \sum_{i=1}^t\bphi_i^{\top}(\Lambda_t)^{-1}\bphi_i \leq d.
    \end{align*}
   
\end{lemma}

\begin{proof}[Proof of \Cref{lemma:Bound of weights}]
    Denote $\alpha_i = \argmax_{\alpha\in[0,H]} \{z^{k}_{h,i}(\alpha)-\rho\alpha\}, i\in[d]$. 
    For any vector $\bv \in \RR^d$, we have 
    \begin{align}
        \big|\bv^{\top}\bw_h^{\rho, k}\big| &= \Bigg|\bv^{\top}\btheta_h + \bv^{\top} \bigg[\max_{\alpha\in[0,H]}\{z^{k}_{h,i}(\alpha)-\rho\alpha\} \bigg]_{i\in [d]} \Bigg| \notag \\
        &\leq \big|\bv^{\top}\btheta_h\big| + \Bigg|\bv^{\top} \bigg[\max_{\alpha\in[0,H]}\{z^{k}_{h,i}(\alpha)-\rho\alpha\} \bigg]_{i\in [d]} \Bigg| \notag \\
        &\leq \sqrt{d}\|\bv\|_2 + H\Vert\bv\Vert_1 + \Bigg|\bv^{\top}\bigg[\bigg((\Lambda_h^k)^{-1}\sum_{\tau=1}^{k-1}\bphi_h^{\tau}[\max_a Q_{h+1}^{k, \rho}(s_{h+1}^{\tau},a)]_{\alpha_i}\bigg)_{i}\bigg]_{i \in [d]}
        \Bigg| \label{eq:expand_z_h_i}\\
        &\leq \sqrt{d}\|\bv\|_2 + H\sqrt{d}\Vert \bv\Vert_2 + \sqrt{\bigg[ \sum_{\tau=1}^{k-1}\bv^{\top}(\Lambda_h^k)^{-1}\bv\bigg]\bigg[\sum_{\tau=1}^{k-1}(\bphi_h^{\tau})^{\top}(\Lambda_h^k)^{-1}(\bphi_h^{\tau})\bigg]}\cdot H \label{eq:use_C_S}\\
        &\leq 2H\Vert\bv\Vert_2\sqrt{dk/\lambda}\label{eq:weight_bound}.
    \end{align}
    We note that the term $[((\Lambda_h^k)^{-1}\sum_{\tau=1}^{k-1}\bphi_h^{\tau}[\max_a Q_{h+1}^{k, \rho}(s_{h+1}^{\tau},a)]_{\alpha_i})_{i}]_{i \in [d]}$ in \eqref{eq:expand_z_h_i} is constructed by first taking out the $i$-th coordinate of the ridge solution vector, $(\Lambda_h^k)^{-1}\sum_{\tau=1}^{k-1}\bphi_h^{\tau}[\max_a Q_{h+1}^{k, \rho}(s_{h+1}^{\tau},a)]_{\alpha_i}\in\RR^d,~\forall i\in[d]$, and then concatenating all $d$ values into a vector.    
    Inequality \eqref{eq:expand_z_h_i} is due to the fact that $\rho \leq 1$, \eqref{eq:use_C_S} is due to the fact that $Q_h^{k, \rho} \leq H$, and   
    \eqref{eq:weight_bound} is due to \Cref{lemma:self-normalize} and the fact that the minimum eigenvalue of $\Lambda_h^k$ is lower bounded by $\lambda$. The remainder of the proof follows from the fact that $\Vert \bw_h^{\rho, k} \Vert_2 = \max_{\bv:\Vert\bv\Vert_2=1}|\bv^{\top}\bw_h^{\rho, k}| $.
\end{proof}

\subsection{Proof of \Cref{lemma:Uniform concentration bound}}
The proof of \Cref{lemma:Uniform concentration bound} requires the following results on the concentration of self-normalized processes.
\begin{lemma}[Concentration of Self-Normalized Processes]
\cite[Theorem 1]{abbasi2011improved}\label{th:Concentration of Self-Normalized Processes}
    Let $\{\epsilon_t\}_{t=1}^{\infty}$ be a real-valued stochastic process with corresponding filtration $\{\mathcal{F}_t\}_{t=0}^{\infty}$. Let $\epsilon_t|\mathcal{F}_{t-1}$ be mean-zero and $\sigma$-subGaussian; i.e. $\mathbb{E}[\epsilon_t|\mathcal{F}_{t-1}]=0$, and 
    \begin{equation*}
        \forall \lambda \in \mathbb{R}, ~~~~\mathbb{E}[e^{\lambda \epsilon_t}|\mathcal{F}_{t-1}] \leq e^{\lambda^2\sigma^2/2}.
    \end{equation*}
    Let $\{\bm{\phi}_t\}_{t=1}^{\infty}$ be an $\mathbb{R}^d$-valued stochastic process where $\phi_t$ is $\mathcal{F}_{t-1}$ measurable. Assume $\Lambda_0$ is a $d\times d$ positive definite matrix, and let $\Lambda_t=\Lambda_0+\sum_{s=1}^t\bm{\phi}_s\bm{\phi}_s^\top$. Then for any $\delta > 0$, with probability at least $1-\delta$, we have for all $t \geq 0$:
    \begin{equation*}
        \bigg\Vert \sum_{s=1}^t \bm{\phi}_s\epsilon_s \bigg\Vert^2_{\Lambda_t^{-1}} 
        \leq 2\sigma^2 \log \bigg[ \frac{\det(\Lambda_t)^{1/2}\det(\Lambda_0)^{-1/2}}{\delta}\bigg].
    \end{equation*}
\end{lemma}

\begin{proof}[Proof of \Cref{lemma:Uniform concentration bound}]
    For any $V\in \cV$ and $\alpha\in[0,H]$, we know there exists a $\tilde{\alpha}$ in the $\epsilon_1$-covering and a $\tilde{V}$ in the $\epsilon_2$ covering such that 
    \begin{align*}
        &V=\tilde{V}+\Delta_V, ~\sup_x|\Delta_V(x)|\leq \epsilon,\\
        &\alpha = \tilde{\alpha}+\Delta_{\alpha}, ~|\Delta_{\alpha}|\leq \epsilon.
    \end{align*}
    This gives the following decomposition:
    \begin{align*}
        &\bigg\Vert \sum_{\tau=1}^k\bphi_{\tau}\big\{\big[V(x_{\tau})\big]_{\alpha}-\EE\big[\big[V(x_{\tau})\big]_{\alpha} |\cF_{\tau-1}\big]\big\}\bigg\Vert^2_{\Lambda_k^{-1}} \\
        &\leq 2\bigg\Vert \sum_{\tau=1}^k\bphi_{\tau}\big\{\big[\tilde{V}(x_{\tau})\big]_{\alpha}-\EE\big[\big[\tilde{V}(x_{\tau})\big]_{\alpha} |\cF_{\tau-1}\big]\big\}\bigg\Vert^2_{\Lambda_k^{-1}}\\
        &\qquad +2\bigg\Vert \sum_{\tau=1}^k\bphi_{\tau}\big\{\big[V(x_{\tau})\big]_{\alpha}- \big[\tilde{V}(x_{\tau})\big]_{\alpha}-\EE\big[\big[V(x_{\tau})\big]_{\alpha} -\big[\tilde{V}(x_{\tau})\big]_{\alpha}|\cF_{\tau-1}\big]\big\}\bigg\Vert^2_{\Lambda_k^{-1}}\\
        &\leq 4\bigg\Vert \sum_{\tau=1}^k\bphi_{\tau}\big\{\big[\tilde{V}(x_{\tau})\big]_{\tilde{\alpha}}-\EE\big[\big[\tilde{V}(x_{\tau})\big]_{\tilde{\alpha}} |\cF_{\tau-1}\big]\big\}\bigg\Vert^2_{\Lambda_k^{-1}}\\
        &\qquad +4\bigg\Vert \sum_{\tau=1}^k\bphi_{\tau}\big\{\big[\tilde{V}(x_{\tau})\big]_{\alpha}- \big[\tilde{V}(x_{\tau})\big]_{\tilde{\alpha}}-\EE\big[\big[\tilde{V}(x_{\tau})\big]_{\alpha} -\big[\tilde{V}(x_{\tau})\big]_{\tilde{\alpha}}|\cF_{\tau-1}\big]\big\}\bigg\Vert^2_{\Lambda_k^{-1}}\\
        &\qquad +2\bigg\Vert \sum_{\tau=1}^k\bphi_{\tau}\big\{\big[V(x_{\tau})\big]_{\alpha}- \big[\tilde{V}(x_{\tau})\big]_{\alpha}-\EE\big[\big[V(x_{\tau})\big]_{\alpha} -\big[\tilde{V}(x_{\tau})\big]_{\alpha}|\cF_{\tau-1}\big]\big\}\bigg\Vert^2_{\Lambda_k^{-1}}.
    \end{align*}
   We can apply \Cref{th:Concentration of Self-Normalized Processes} and a union bound to the first term, and the second and the third term can be bounded by $16k^2\epsilon_1^2/\lambda$ and $8k^2\epsilon_2^2/\lambda$, respectively. Therefore we complete the proof.
\end{proof}

\subsection{Proof of \Cref{lemma:Covering number of the function class V}}
The proof of \Cref{lemma:Covering number of the function class V} will use the following fact.
\begin{lemma}
\label{lemma:Covering Number of Euclidean Ball}
    \cite[Covering Number of Euclidean Ball]{jin2020provably} For any $\epsilon>0$, the $\epsilon$-covering number of the Euclidean ball in $\RR^d$ with radius $R > 0$ is upper bounded by $(1 + 2R/\epsilon)^d$. 
\end{lemma}

\begin{proof}[Proof of \Cref{lemma:Covering number of the function class V}]
    The argument is similar to the proof of Lemma D.6 in \cite{jin2020provably}. Denote $A=\beta^2\Lambda^{-1}$, so we have
    \begin{align}
    \label{eq:function class V}
        V(\cdot)=\min \bigg\{\max_{a}\bigg\{\wb^{\top}\bphi(\cdot,a)+\sum_{i=1}^d\sqrt{\phi_i(\cdot,a)\mathbf{1}_i^{\top}A\phi_i(\cdot,a)\mathbf{1}_i}\bigg\},H \bigg\},
    \end{align}
    for $\Vert w\Vert \leq L$ and $\Vert A \Vert\leq B^2\lambda^{-1}$. For any two functions $V_1, V_2 \in \cV$, let them take the form in \eqref{eq:function class V} with parameters $(\wb_1, A_1)$ and $(\wb_2,A_2)$, respectively. Then since both $\min\{\cdot, H\}$ and $\max_a$ are contraction maps, we have %
    \begin{align}
        \dist(V_1, V_2)&\leq \sup_{x,a}\bigg|\bigg[\wb_1^{\top}\bphi(x,a)+\sum_{i=1}^d\sqrt{\phi_i(x,a)\mathbf{1}_i^{\top}A_1\phi_i(x,a)\mathbf{1}_i} \bigg] \notag - \bigg[\wb_2^{\top}\bphi(x,a)+\sum_{i=1}^d\sqrt{\phi_i(x,a)\mathbf{1}_i^{\top}A_2\phi_i(x,a)\mathbf{1}_i} \bigg]\bigg|\notag\\
        &\leq \sup_{\bphi:\Vert\bphi\Vert\leq 1}\bigg|\bigg[\wb_1^{\top}\bphi+\sum_{i=1}^d\sqrt{\phi_i\mathbf{1}_i^{\top}A_1\phi_i\mathbf{1}_i} \bigg] - \bigg[\wb_2^{\top}\bphi+\sum_{i=1}^d\sqrt{\phi_i\mathbf{1}_i^{\top}A_2\phi_i\mathbf{1}_i} \bigg] \bigg|\notag\\
        &\leq \sup_{\bphi:\Vert\bphi\Vert\leq 1}\big|(\wb_1-\wb_2)^{\top}\bphi\big|  + \sup_{\bphi:\Vert\bphi\Vert\leq 1}\sum_{i=1}^d\sqrt{\phi_i\mathbf{1}_i^{\top}(A_1-A_2)\phi_i\mathbf{1}_i}\label{eq:bound_by_triangular}\\
        &\leq \Vert \wb_1-\wb_2\Vert + \sqrt{\Vert A_1-A_2\Vert}\sup_{\bphi:\Vert\bphi\Vert\leq 1}\sum_{i=1}^d\Vert \phi_i\mathbf{1}_i \Vert\notag\\
        & \leq \Vert \wb_1-\wb_2\Vert + \sqrt{\Vert A_1-A_2\Vert_F}, \label{eq:dist(V1,V2)}
    \end{align}
    where \eqref{eq:bound_by_triangular} follows from triangular inequlaity and the fact that $|\sqrt{x} - \sqrt{y}| \leq \sqrt{|x-y|},~ \forall x, y
    \geq 0$. For matrices, $\Vert\cdot\Vert$ and $\Vert\cdot\Vert_F$ denote the matrix operator norm and Frobenius norm respectively.

    Let $\cC_{\wb}$ be an $\epsilon/2$-cover of $\{\wb\in\RR^d|\Vert\wb\Vert_2\leq L\}$ with respect to the 2-norm, and $\cC_{A}$ be an $\epsilon^2/4$-cover of $\{A\in\RR^{d\times d}|\Vert A\Vert_F\leq d^{1/2}B^2\lambda^{-1}\}$ with respect to the Frobenius norm. By \Cref{lemma:Covering Number of Euclidean Ball}, we know:
    \begin{align*}
        \big|\cC_{\wb}\big|\leq \big(1+4L/\epsilon\big)^d, \quad \big|\cC_A\big|\leq \big[1+8d^{1/2}B^2/(\lambda\epsilon^2)\big]^{d^2}.
    \end{align*}
    By \eqref{eq:dist(V1,V2)}, for any $V_1\in \cV$, there exists $\wb_2\in \cC_{\wb}$ and $A_2\in\cC_A$ such that $V_2$ parametrized by $(\wb_2, A_2)$ satisfies $\dist(V_1, V_2)\leq \epsilon$. Hence, it holds that $\cN_{\epsilon}\leq |\cC_{\wb}|\cdot|\cC_{A}|$, which leads to
    \begin{align*}
        \log\cN_{\epsilon}\leq \log|\cC_{\wb}|+\log|\cC_A| \leq d\log(1+4L/\epsilon) + d^2\log\big[1+8d^{1/2}B^2/(\lambda\epsilon^2)\big].
    \end{align*}
    This concludes the proof.
\end{proof}

\subsection{Proof of \Cref{lemma:Bounded difference}}
\begin{proof}
    For all $(s,a,h)\in\cS/\{s_f\}\times\cA\times[H]$, we have
    \begin{align*}
        Q_h^{\pi, \rho}(s,a) = \la\bphi(s,a), \btheta_h+\bnu_h^{\pi, \rho} \ra = r_h(s,a) + \inf_{P_h(\cdot|s,a) \in \cU_{h}^{\rho}(s,a;\bmu_h^0)} [\PP_hV_{h+1}^{\pi,\rho}](s,a).
    \end{align*}
    This gives
    \begin{align}\label{eq:decomp_model_error}
        (\btheta_h+\bnu_{h}^{\rho, k}) - (\btheta_h+\bnu_{h}^{\pi, \rho}) = \bnu_{h}^{\rho, k} - \bnu_{h}^{\pi, \rho} = \underbrace{\bnu_{h}^{\rho, k} - \tilde{\bnu}_h^{k, \rho}}_{I}+ \underbrace{\tilde{\bnu}_h^{k, \rho} - \bnu_{h}^{\pi, \rho}}_{II},
    \end{align}
    where $\tilde{\bnu}_h^{k, \rho} = \big[\tilde{\nu}_{h,i}^{k, \rho}\big]_{i\in[d]}$, and $\tilde{\nu}_{h,i}^{k, \rho} = \max_{\alpha \in [0,H]}\big\{\EE^{\mu_{h,i}^0} \big[V_{h+1}^{k, \rho}(s)\big]_{\alpha}-\rho\alpha\big\}$. In what follows, we will bound these two terms separately.
    
    For term $I$ in \eqref{eq:decomp_model_error}, we have
    \begin{align*}
        &\bnu_{h}^{\rho, k} - \tilde{\bnu}_h^{k, \rho} \leq \Big[\max_{\alpha \in [0,H]} \Big\{ \widehat{\EE^{\mu_{h,i}^0} 
    \Big[V_{h+1}^{k, \rho}(s)\Big]_{\alpha}} - \EE^{\mu_{h,i}^0} \Big[V_{h+1}^{k, \rho}(s)\Big]_{\alpha} \Big\}\Big]_{i\in[d]}.
    \end{align*}
    Denote $\alpha_i^k = \argmax_{\alpha \in [0,H]}\big\{ \widehat{\EE^{\mu_{h,i}^0} \big[V_{h+1}^{k, \rho}(s)\big]_{\alpha}} - \EE^{\mu_{h,i}^0} \big[V_{h+1}^{k, \rho}(s)\big]_{\alpha}\big\},~i=1,\cdots,d$. Then we have
    \begin{align}
         &\bnu_{h}^{\rho, k} - \tilde{\bnu}_h^{k, \rho} \notag \\
         &\leq \bigg[\bigg(\big(\Lambda_h^k\big)^{-1}\sum_{\tau=1}^{k-1}\bphi_h^{\tau}\bigg[V_{h+1}^{k, \rho}(s_{h+1}^{\tau})\bigg]_{\alpha_i^k} \bigg)_i -   \bigg(\EE^{\bmu_{h}^0} \Big[V_{h+1}^{k, \rho}(s)\Big]_{\alpha_i^k} \bigg)_i \bigg]_{i\in[d]}\notag \\
        &=\bigg[\bigg(-\lambda\big(\Lambda_h^k\big)^{-1} \EE^{\bmu_{h}^0} \Big[V_{h+1}^{k, \rho}(s)\Big]_{\alpha_i^k}\bigg)_i + \bigg(\big(\Lambda_h^k\big)^{-1}\sum_{\tau=1}^{k-1}\bphi_h^{\tau}\bigg[\Big[V_{h+1}^{k, \rho}(s_{h+1}^{\tau})\Big]_{\alpha_i^k} - \Big[\PP_h^0\Big[V_{h+1}^{k, \rho}\Big]_{\alpha_i^k}\Big](s_h^{\tau}, a_h^{\tau}) \bigg] \bigg)_i\bigg]_{i\in[d]}\label{eq:difference decomposition}.
    \end{align}
    For the first term on the RHS of \eqref{eq:difference decomposition}, 
    \begin{align}
        &\bigg|\bigg\la \bphi(s,a), \bigg[\bigg(-\lambda\big(\Lambda_h^k\big)^{-1} \EE^{\bmu_{h}^0} \Big[V_{h+1}^{k, \rho}(s)\Big]_{\alpha_i^k}\bigg)_i \bigg]_{i\in[d]} \bigg\ra \bigg|\notag \\
        &=\bigg|\sum_{i=1}^d\phi_i(s,a)\mathbf{1}_i^{\top}(-\lambda)\big(\Lambda_h^k\big)^{-1}\EE^{\bmu_h^0}\Big[V_{h+1}^{k, \rho}(s)\Big]_{\alpha_i^k}  \bigg|\notag \\
        &\leq \lambda \sum_{i=1}^d\sqrt{\phi_i(s,a)\mathbf{1}_i^{\top}\big(\Lambda_h^k\big)^{-1}\phi_i(s,a)\mathbf{1}_i} \cdot \bigg\Vert \EE^{\bmu_h^0}\Big[V_{h+1}^{k, \rho}(s) \Big]_{\alpha_i^k}\bigg\Vert_{(\Lambda_h^k)^{-1}}\notag\\
        & \leq \sqrt{\lambda}H\sum_{i=1}^d\sqrt{\phi_i(s,a)\mathbf{1}_i^{\top}\big(\Lambda_h^k\big)^{-1}\phi_i(s,a)\mathbf{1}_i},\label{eq:first term bound}
    \end{align}
    where $\mathbf{1}_i$ is the vector with the $i$-th entry being 1 and else being 0. The first inequality holds due to the Cauchy-Schwarz inequality.
    For the second term on the RHS of \eqref{eq:difference decomposition}, given the event $\cE$ defined in \Cref{lemma:Concentration} we have,
    \begin{align}
        &\bigg|\bigg\la \bphi(s,a), \bigg[ \bigg((\Lambda_h^k)^{-1}\sum_{\tau=1}^{k-1}\bphi_h^{\tau}\bigg[\Big[V_{h+1}^{k, \rho}(s_{h+1}^{\tau})\Big]_{\alpha_i^k} - \Big[\PP_h^0 \Big[V_{h+1}^{k, \rho}\Big]_{\alpha_i^k}\Big](s_h^{\tau}, a_h^{\tau}) \bigg] \bigg)_i \bigg]_{i\in[d]} \bigg\ra \bigg|\notag\\
        & = \bigg|\sum_{i=1}^d\phi_i(s,a) \mathbf{1}_i^{\top}(\Lambda_h^k)^{-1}\sum_{\tau=1}^{k-1}\bphi_h^{\tau}\bigg[\Big[V_{h+1}^{k, \rho}(s_{h+1}^{\tau})\Big]_{\alpha_i^k} - \Big[\PP_h^0 \Big[V_{h+1}^{k, \rho}\Big]_{\alpha_i^k}\Big](s_h^{\tau}, a_h^{\tau})\bigg]\bigg| \notag\\
        & \leq \sum_{i=1}^d\sqrt{\phi_i(s,a)\mathbf{1}_i^{\top}\big(\Lambda_h^k\big)^{-1}\phi_i(s,a)\mathbf{1}_i} \cdot \bigg\Vert \sum_{\tau=1}^{k-1}\bphi_h^{\tau}\bigg[\Big[V_{h+1}^{k,\rho}(s_{h+1}^{\tau})\Big]_{\alpha_i^k}-\Big[\PP_h^0 \Big[V_{h+1}^{k, \rho}\Big]_{\alpha_i^k}\Big](s_h^{\tau}, a_h^{\tau}) \bigg]\bigg\Vert_{(\Lambda_h^k)^{-1}}\notag\\
        & \leq C\cdot dH\sqrt{\chi}\sum_{i=1}^d\sqrt{\phi_i(s,a)\mathbf{1}_i^{\top}\big(\Lambda_h^k\big)^{-1}\phi_i(s,a)\mathbf{1}_i},\label{eq:second term bound}
    \end{align}
    for an absolute constant $C$ independent of $c_{\beta}$, and $\chi = \log[3(c_{\beta}+1)dT/p]$.
    Combining \eqref{eq:difference decomposition}, \eqref{eq:first term bound} and \eqref{eq:second term bound}, we have
    \begin{align*}
        \big\la\bphi(s,a), \bnu_{h}^{\rho, k} - \tilde{\bnu}_h^{k, \rho} \big\ra \leq c\cdot dH\sqrt{\chi}\sum_{i=1}^d\sqrt{\phi_i(s,a)\mathbf{1}_i^{\top}\big(\Lambda_h^k\big)^{-1}\phi_i(s,a)\mathbf{1}_i},
    \end{align*}
    for an absolute constant $c$ independent of $c_\beta$.
    On the other hand, we can similarly deduce 
    $
        \la\bphi(s,a),  \tilde{\bnu}_h^{k, \rho}-\bnu_{h}^{\rho, k}  \ra \leq c\cdot dH\sqrt{\chi}\sum_{i=1}^d\sqrt{\phi_i(s,a)\mathbf{1}_i^{\top}(\Lambda_h^k)^{-1}\phi_i(s,a)\mathbf{1}_i}.
    $
    Thus, we have
    \begin{align}
    \label{eq:bound on estimation error}
        \big|\la\bphi(s,a), \bnu_{h}^{\rho, k} - \tilde{\bnu}_h^{k, \rho} \ra\big|\leq c\cdot dH\sqrt{\chi}\sum_{i=1}^d\sqrt{\phi_i(s,a)\mathbf{1}_i^{\top}\big(\Lambda_h^k\big)^{-1}\phi_i(s,a)\mathbf{1}_i}.
    \end{align}
    
    For term $II$ in \eqref{eq:decomp_model_error}, we have
    \begin{align*}
        \big \la \bphi(s,a), \tilde{\bnu}_h^{k, \rho} - \bnu_h^{\pi, \rho}\big \ra = \inf_{P_h(\cdot|s,a) \in \cU_{h}^{\rho}(s,a;\bmu_h^0)}\Big[\PP_hV_{h+1}^{k, \rho}\Big](s,a) - \inf_{P_h(\cdot|s,a) \in \cU_{h}^{\rho}(s,a;\bmu_h^0)}\Big[\PP_hV_{h+1}^{\pi, \rho}\Big](s,a).
    \end{align*}
    Finally, since $\la \bphi(s,a), \btheta_h+\bnu_{h}^{\rho, k}\ra - Q_h^{\pi, \rho}(s,a) =  \la\bphi(s,a),\bnu_{h}^{\rho, k} - \tilde{\bnu}_{h}^{k, \rho} + \tilde{\bnu}_{h}^{k, \rho} -\bnu_{h}^{\pi, \rho}\ra$, by our choice of $\lambda$ and \eqref{eq:bound on estimation error} we have 
    \begin{align*}
        &\bigg|\big\la \bphi(s,a), \btheta_h+\bnu_{h}^{\rho, k}\big\ra - Q_h^{\pi, \rho}(s,a) - \bigg(\inf_{P_h(\cdot|s,a) \in \cU_{h}^{\rho}(s,a;\bmu_h^0)}\Big[\PP_hV_{h+1}^{k, \rho}\Big](s,a) - \inf_{P_h(\cdot|s,a) \in \cU_{h}^{\rho}(s,a;\bmu_h^0)}\Big[\PP_hV_{h+1}^{\pi, \rho}\Big](s,a)\bigg) \bigg|
        \\
        & = \big|\big \la \bphi(s,a), \bnu_{h}^{\rho, k} - \tilde{\bnu}_h^{k, \rho}\big \ra \big|
        \\
        &\leq c\cdot dH\sqrt{\chi} \sum_{i=1}^d\sqrt{\phi_i(s,a)\mathbf{1}_i^{\top}\big(\Lambda_h^k\big)^{-1}\phi_i(s,a)\mathbf{1}_i}.
    \end{align*}

Finally, to prove this lemma, we only need to show that there exists a choice of absolute value $c_{\beta}$ so that 
\begin{align}
\label{eq:choice of absolute constant}
c'\sqrt{\iota+\log(c_{\beta}+1)}\leq c_{\beta}\sqrt{\iota},
\end{align}
where $\iota=\log3dT/p$. We know $\iota\in [\log3, \infty)$ by its definition, and $c'$ is an absolute constant independent of $c_{\beta}$. Therefore we can pick an absolute constant $c_{\beta}$ which satisfies $c'\sqrt{\log3+\log(c_{\beta}+1)}\leq c_{\beta}\sqrt{\log3}$. This choice of $c_{\beta}$ will ensure \eqref{eq:choice of absolute constant} hold for all $\iota \in [\log3, \infty)$, which finishes the proof.
\end{proof}

\end{document}